\documentclass[twoside]{article}

\usepackage[accepted]{aistats2020}
\usepackage{mathtools}
\usepackage[toc,page]{appendix}
\usepackage{comment}
\usepackage{times}
\usepackage{epsfig}
\usepackage{graphicx}
\usepackage{amsmath}
\usepackage{caption}
\usepackage{subcaption}
\usepackage{amssymb}
\usepackage{amsthm}
\usepackage{ifthen}
\usepackage[linesnumbered, ruled,vlined]{algorithm2e}
\usepackage{caption,booktabs,array}
\usepackage{multirow}
\usepackage{dsfont}
\usepackage{tikz}
\usepackage{xcolor}

\newcommand{\blue}{\color{blue}}

\newcommand\V{\vert}
\renewcommand{\a}{\alpha}

\newcommand{\Pa}{\mathcal{P}_m(\alpha_n)}
\newcommand{\Pb}{\mathcal{P}_m(\beta_n)}
\newcommand{\an}{\alpha _n}
\newcommand{\bn}{\beta _n}

\usepackage[utf8]{inputenc}

\def\xx{{\boldsymbol x}}
\def\yy{{\boldsymbol y}}

\def\UU{{\boldsymbol U}}

\def\XX{{\boldsymbol X}}

\def\ZZ{{\boldsymbol Z}}
\def\R{{\mathbb{R}}}
\def\PP{{\boldsymbol P}}
\def\balpha{{\boldsymbol \alpha}}

\def\bbeta{{\boldsymbol \beta}}

\def\expect{\mathop{\mathbb{E}}}
\newtheorem{definition}{Definition}
\newtheorem{theorem}{Theorem}
\newtheorem{remark}{Remark}
\newtheorem{lemma}{Lemma}
\newtheorem{corollary}{Corollary}
\newtheorem{proposition}{Proposition}
\usepackage[pagebackref=true,breaklinks=true,letterpaper=true,bookmarks=false]{hyperref}

\newcommand{\rf}[1]{{\color{blue} #1}}

\newcommand{\rg}[1]{{\color{orange} #1}}
\newcommand{\kf}[1]{{\color{green} #1}}
\newcommand{\defas}{\;\mathrel{\!\!{:}{=}\,}}
\newcommand{\fact}[1]{#1\mathpunct{}!}

\begin{document}
\def\httilde{\mbox{\tt\raisebox{-.5ex}{\symbol{126}}}}


\twocolumn[

\aistatstitle{Learning with minibatch Wasserstein  : asymptotic and gradient properties} %Mini batch Wasserstein: A review of mini batch Optimal Transport distance}

\aistatsauthor{ Kilian Fatras\footnotemark[1] \And Younes Zine\footnotemark[2] \And R\'emi Flamary\footnotemark[3] \And R\'emi Gribonval\footnotemark[2]\footnotemark[4] \And Nicolas Courty\footnotemark[1]}

\aistatsaddress{ \\ \footnotemark[1]Univ Bretagne Sud, Inria, CNRS, IRISA, France  \\  \footnotemark[2]Univ Rennes, Inria, CNRS, IRISA, France \\ \footnotemark[3]Univ C\^ote d’Azur, OCA, UMR 7293, CNRS, Laboratoire Lagrange, France \\ \footnotemark[4]Univ Lyon, Inria, CNRS, ENS de Lyon, UCB Lyon 1, LIP UMR 5668, F-69342, Lyon, France } ]

\begin{abstract}
Optimal transport distances are powerful tools to compare probability distributions and have found many applications in machine learning. Yet their algorithmic complexity prevents their direct use on large scale datasets. To overcome this challenge, practitioners compute these distances on minibatches {\em i.e.} they average the outcome of several smaller optimal transport problems. We propose in this paper an analysis of this practice, which effects are not well understood so far. We notably argue that it is equivalent to an implicit regularization of the original problem, with appealing properties such as unbiased estimators, gradients and a concentration bound around the expectation, but also with defects such as loss of distance property. Along with this theoretical analysis, we also conduct empirical experiments on gradient flows, GANs or color transfer that highlight the practical interest of this strategy.
%convergence in population independent from data space dimensionality
%Wasserstein distance has a cubical complexity and it makes it computationally challenging. To overcome this challenge, practitioners rely on computing the Wasserstein distance on minibatches. While it replaces the original problem, it has been effective in practice for domain adaptation and generative modeling tasks. In this paper we propose a deeper study of the minibatch optimal transport paradigm. We show that it comes with inherited properties from OT and has a mass spread behavior similar to regularized Wasserstein distance variants. It also has properties which are not shared with OT losses: unbiased estimator, asymptotic convergence without dependence on dimension and unbiased gradients. But using minibatches induce a bias which break the first distance axiom. To study their behavior, we considered gradient flow experiments on the CelebA dataset, a toy GAN example and color transfert experiments between 1M pixel images.
\end{abstract}

\section{Introduction}

Measuring distances between probability distributions is a key problem in machine learning.
%It has been used in several problems such as, among many others, generative modeling \cite{goodfellow2014}, domain adaptation \cite{DACourty} and classification \cite{frogner_2015}.
Considering the space of probability distributions $\mathcal{M}_{1}^{+}(\mathcal{X})$ over a space $\mathcal{X}$, and given an empirical probability distribution $\alpha \in \mathcal{M}_{1}^{+}(\mathcal{X})$, we want to find a parametrized distribution $\beta_{\lambda}$ which approximates the distribution $\alpha$. Measuring the distance between the distributions requires a function $L : \mathcal{M}_{1}^{+}(\mathcal{X}) \times \mathcal{M}_{1}^{+}(\mathcal{X}) \to \mathbb{R}$. The distribution $\beta$ is parametrized by a vector $\lambda$ and the goal is to find the best $\lambda$ which minimizes the distance $L$ between $\beta_{\lambda}$ and $\alpha$, i.e $L(\alpha, \beta_{\lambda})$. As the distributions are empirical, we need a distance  $L$ with good statistical performance and which have optimization guarantees with modern optimization techniques. Optimal transport (OT) losses as distances have emerged recently as a competitive tool on this problem \cite{genevay_2018, arjovsky_2017}. The corresponding estimator is usually found in the literature under the name of {\em Minimum Kantorovich Estimator}~\cite{Bassetti06,COT_Peyre}. Furthermore, OT losses have been widely used to transport samples from a source domain to a target domain using barycentric mappings \cite{Ferradans2013, DACourty,seguy2018large}. %However, the barycentric mapping strategy can not be used when the size of distribution supports is too important, i.e big data setting, because of memory limitations.

Several previous works challenged the heavy computational cost of optimal transport, as the Wasserstein distance comes with a complexity of $\mathcal{O}(n^3log(n))$, where $n$ is the size of the probability distribution supports. Variants of optimal transport have been proposed to reduce its complexity. \cite{CuturiSinkhorn} used an entropic regularization term to get a strongly convex problem which is solvable using the Sinkhorn algorithm with a computational cost of $\mathcal{O}(n^2)$, both in time and space. However, despite some scalable solvers based on stochastic optimization \cite{genevay2016stochastic,seguy2018large}, in the big data setting $n$ is very large and still leads to bottleneck computation problems especially when trying to minimize the OT loss. That is why \cite{genevay_2018, deepjdot} use a minibatch strategy in their implementations to reduce the cost per iteration. They propose to compute the averaged of several optimal transport terms between minibatches from the source and the target distributions. However, using this strategy leads to a different optimization problem that results in a "non optimal" transportation plan between the full original distributions. Recently, \cite{bernton2017} worked on minimizers and \cite{mbot_Sommerfeld} on a bound between the true optimal transport and the minibatch optimal transport. However they did not study the asymptotic convergence, the loss properties and behavior of the minibatch loss.
%Recently, \cite{bernton2017} showed the convergence of the minibatch minimizers to the true minimizers when the batch size increases and \cite{mbot_Sommerfeld} showed a bound between the true optimal transport and the minibatch optimal transport. However they did not study the asymptotic convergence, the loss properties and behavior of the minibatch loss.

In this paper we propose to study  minibatch optimal transport by reviewing its relevance as a loss function. After defining the minibatch formalism, we will show which properties are inherited and which ones are lost. %and the consequences to consider the minibatch on the transportation plan.
We describe the asymptotic behavior of the estimator and show that we can derive a concentration bound without dependence on the data space dimension.
%contrary to Wasserstein distance which suffers from the curse of dimensionality, it does not depend on the dimension.
Then, we prove that the gradients of the minibatch OT losses are unbiased, which justifies its use with SGD in \cite{genevay_2018}. Finally, we  demonstrate the effectiveness of minibatches in large scale setting and show how to alleviate the memory issues for barycentric mapping. The paper is structured as follows: in Section 2, we propose a brief review of the different optimal transport losses. In Section 3, we give formal definitions of the minibatch strategy and illustrate their impacts on OT plans. Basic properties, asymptotic behaviors of the estimator and differentiability are then described. Finally in Section 4, we highlight the behavior of the minibatch OT losses on a number of experiments: gradient flows, generative networks and color transfer.

\section{Wasserstein distance and regularization}
\begin{comment}
\paragraph{MMD losses}
{\blue useful,exp?} To consider the distribution geometries, one can consider maximum mean discrepancy (MMD) losses \cite{MMD_Gretton}. They are integral probability metrics which consist to integrate a positive kernel $k$ over a reproducing kernel Hilbert space $\mathcal{X}$. The loss is defined as:
\begin{align}
\mathrm{L}_{k}(\alpha, \beta) \stackrel{\mathrm{def.}}{=} &\mathbb{E}_{\balpha \otimes \balpha}\left[k\left(X, X^{\prime}\right)\right] + \mathbb{E}_{\bbeta \otimes \bbeta}\left[k\left(Y, Y^{\prime}\right)\right]  \nonumber \\
& \quad- 2 \mathbb{E}_{\balpha \otimes \bbeta}[k(X, Y)]
\end{align}

MMD losses have an appealing property. Their sample complexity, the convergence rate
of a given metric between a measure and its empirical counterpart, is free of dimension, scaling as $\mathcal{O}(n^{-1/2})$. MMD losses have been successfully used for generative modeling {\blue KF: cite}.
\end{comment}

\paragraph{Wasserstein distance}
%Details on OT losses
The Optimal Transport metric measures a distance between two probability distributions $(\balpha, \bbeta) \in \mathcal{M}_{+}^1(\mathcal{X}) \times \mathcal{M}_{+}^1(\mathcal{X})$ by considering a ground metric $c$ on the space $\mathcal{X}$ \cite{COT_Peyre}. % It reaches its goal by minimizing the displacement cost of $\balpha$ to
%$\bbeta$ with respect to the ground cost $\CC$.
Formally, the Wasserstein distance between two distributions can be expressed as
\begin{equation}
    W_{c}(\balpha, \bbeta) = \underset{\pi \in \UU(\balpha, \bbeta)}{\text{min}} \int_{\mathcal{X}\times \mathcal{Y}}c(\xx,\yy) d\pi(\xx,\yy),
\label{eq:wasserstein_dist}
\end{equation}
%Where $\langle ., . \rangle$ is the Frobenius product.
where $\UU(\balpha, \bbeta)$ is the set of joint probability distribution with
marginals $\balpha$ and $\bbeta$ such that %\newline
$
\boldsymbol U(\balpha, \bbeta) = \left \{ \pi \in \mathcal{M}_{+}^1(\mathcal{X}, \mathcal{Y}): \PP_{\mathcal{X}}\#\pi = \balpha, \PP_{\mathcal{Y}}\#\pi = \bbeta \right\}\nonumber
$. {$\PP_{\mathcal{X}}\#\pi$ (resp. $\PP_{\mathcal{Y}}\#\pi$) is the marginalization of $\pi$ over $\mathcal{X}$ (resp. $\mathcal{Y}$)}. The ground cost $c(\xx,\yy)$ is usually chosen as as the Euclidean or squared Euclidean distance on $\R^d$, in this case $W_{c}$ is a metric as well. Note that the optimization problem above is called the Kantorovitch formulation of OT and the optimal $\pi$ is called an optimal transport plan. When the distributions are discrete, the problem becomes a discrete linear program that can be solved with a cubic complexity in the size of the distributions support. Also the convergence in population of the Wasserstein distance is known to be slow with a rate  $O(n^{-1/d})$ depending on the dimensionality $d$ of the space $\mathcal{X}$ and the size of the population $n$ \cite{weed2019}. \cite{JMLRGerber} used a multi-scale strategy in order to compute a fast approximation of the Wasserstein distance.  % and is the so-called Wasserstein distance.

%Unfortunately, the OT solution has a cubic complexity and suffers from the curse of dimension, meaning that its sample complexity depends on the dimension on data.
\paragraph{Entropic regularization}
Regularized entropic OT was proposed in \cite{CuturiSinkhorn} and leads to a more efficient $\mathcal{O}(n^2)$ solver. We define the entropic loss as:\newline $W_{c}^{\varepsilon}(\balpha, \bbeta) = \underset{\pi \in \UU(\balpha, \bbeta)}{min} \int\displaylimits_{\mathcal{X}\times\mathcal{Y}}c(\xx, \yy) d\pi(\xx, \yy) + \varepsilon H(\pi|\xi)$, with
$ H(\pi|\xi) = \int_{\mathcal{X}\times\mathcal{Y}} \log(\frac{d\pi(\xx, \yy)}{d\balpha(\xx) d\bbeta(\yy)}(\xx, \yy))d\pi(\xx, \yy)$
where $\xi = \balpha \otimes \bbeta$ and $\varepsilon$ is the regularization coefficient. We call this function, the entropic OT loss. As we will see later, this entropic regularization also makes the problem strongly convex and differentiable with respect to the cost or the input distributions.

It is well known that adding an entropic regularization leads to sub-optimal solutions $\pi$ on the original problem, and it is not a metric since $W_{c}^{\varepsilon}(\bbeta, \bbeta) \ne 0$. This motivated \cite{genevay_2018} to introduce an unbiased loss which uses the entropic regularization and called it the Sinkhorn divergence. It is defined as: \newline$S_{c}^{\varepsilon}(\balpha, \bbeta) = W_{c}^{\varepsilon}(\balpha, \bbeta) - \frac{1}{2}(W_{c}^{\varepsilon}(\balpha, \balpha) + W_{c}^{\varepsilon}(\bbeta, \bbeta))$

%The Sinkhorn divergence fixes the bias by removing the diagonal terms.

It can still be computed with the same order of complexity as the entropic loss and has been proven to interpolate between OT and maximum mean discrepancy (MMD) \cite{feydy19a} with respect to the regularization coefficient. MMD are integral probability metrics over a reproducing kernel Hilbert space \cite{MMD_Gretton}. When $\varepsilon$ tends to 0, we get the OT solution back and when $\varepsilon$ tends to $\infty$, we get a solution closer to the MMD solution. Second, as proved by \cite{feydy19a}, if the cost $c$ is Lipschitz, then $S_{c}^{\varepsilon}$ is a convex, symmetric, positive definite loss function. Hence the use of the Sinkhorn divergence instead of the regularized OT. The sample complexity of the Sinkhorn divergence, that is the convergence rate of a metric between a probability distribution and its empirical counterpart as a function of the number of samples, was proven in  \cite{genevay19} to be:
$
O\left(\frac{e^{\frac{\kappa}{\varepsilon}}}{\sqrt{n}}\left(1+\frac{1}{\varepsilon^{\lfloor d / 2\rfloor}}\right)\right)
$ where $d$ is the dimension of $\mathcal{X}$. We see an interpolation between MMD and OT sample complexity depending on $\varepsilon$.
\paragraph{Minibatch Wasserstein}
While the entropic loss has better computational complexity than the original Wasserstein distance, it is still challenging to compute it for a large dataset. To overcome this issue, several papers rely on a minibatch computation \cite{genevay_2018,deepjdot, liutkus19a, kolouri2016sliced}. Instead of computing the OT problem between the full distributions, they compute an averaged of OT problems between batches of the source and the target domains. It differs from \cite{JMLRGerber} as the size of the minibatch remains constant. Several work came out to justify the minibatch paradigm. \cite{bernton2017} showed that for generative models, the minimizers of the minibatch loss converge to the true minimizer when the minibatch size increases. \cite{mbot_Sommerfeld} considered another approach, where they approximate OT with the minibatch strategy and exhibit a deviation bound between the two quantities. We follow a different approach from the two previous work. We are interested in the behavior of using the minibatch strategy as a loss function. We study the asymptotic behavior of using minibatch, the optimization procedure, the resulting transportation plan and the behavior of such a loss for data fitting problems.

\begin{comment}
\rf{ce paragrephe est un copier collé de l'intro, aussi ajouter les autres refs comme dit en commentaire}
\rg{ecriture un peu compacte au premier abord, c'est peut-être plus pédestre en parlant de comment c'est fait, bref en passant directement à la section 3?}

%\begin{equation}
%\mathbb{E}_{(X,Y) \sim  \balpha^{\otimes m } \otimes \bbeta ^{\otimes m}} [h(X,Y)]
%\end{equation} instead of the original problem.
%%%%%%%%%%%%%%%%%%%%%%%%%%
%Previous work MB (allemand + P.E.Jacob)

\begin{equation}
    \mathcal{OT}_{C}^{\varepsilon}(\balpha, \bbeta) = \underset{\pi \in \UU(\balpha, \bbeta)}{\text{min}} \int_{\mathcal{X}\times\mathcal{Y}}c d\pi + \varepsilon KL(\pi|\balpha \otimes \bbeta)
\end{equation}
With :
\begin{equation}
    KL(\pi|\xi) = \int_{\mathcal{X}x\mathcal{Y}} log(\frac{d\pi}{d\xi}(\xx, \yy))d\pi(\xx, \yy)
\end{equation}

\end{comment}

\section{Minibatch Wasserstein }

In this section we first define the Minibatch Wasserstein and illustrate it on simple examples. Next we study its asymptotic properties and optimization behavior.

\subsection{Notations and Definitions}
\paragraph{Notations}
Let $\boldsymbol{X}=(X_1, \cdots, X_n) $ (resp. $\boldsymbol{Y}=(Y_1, \cdots, Y_n) $) be samples of $n$ \emph{iid} random variables drawn from a distribution $ \alpha $ (resp. $ \beta $) on the source (resp. target) domain. We denote by $\alpha_n$ and $\beta_n$  the empirical distributions of support $ \{X_1, \cdots, X_n \} $ and  $ \{ Y_1, \cdots, Y_n \} $ respectively. The weights of $X_i$ (resp. $Y_i$) are uniform, i.e equal to $1/n$. We further suppose that $\alpha$ and $\beta$ have compact support, the ground cost is then bounded by a constant M. $\alpha^{\otimes m }$ denotes a sample of $m$ random variables following $\alpha$. In the rest of the paper, we will not make a difference between a batch $A$ of cardinality $m$ and its associated (uniform probability) distribution $ \hat{A}:= \frac{1}{m} \sum_{a \in A} \delta_a $.  The number of possible mini-batches of size $m$ on $n$ distinct samples is the binomial coefficient $\dbinom{n}{m}=\frac{\fact{n}}{\fact{m} \fact{(n - m)}}$. For $ 1 \leqslant m \leqslant n  $, we write $ \mathcal{P}_m(\alpha_n) $ (resp. $ \mathcal{P}_m(\beta_n) $) the collection of subsets of cardinality $m$ of $ \alpha_n $ (resp. of $ \beta_n $). We will denote the integer part of the ratio $n/m$ as $\lfloor n/m \rfloor$.

%between any set of size $m$ drawn from $ \alpha $ and any set of size $m$ drawn from $ \beta $ is bound by an absolute constant $ M $%:= d(\operatorname{Supp}(\alpha), \operatorname{Supp}(\beta)) $.

\paragraph{Definitions}

%Our goal is to study the minibatch Wasserstein distance.
We will first give formal definitions of the different quantities that we will use in this paper. We start with minibatch Wasserstein losses for continuous, semi- discrete and discrete distributions.

\begin{definition}[Minibatch Wasserstein definitions] Given an OT loss $h$ and an integer $m \leq n$, we define the following quantities:

The continuous loss:
  \begin{equation}
    U_h(\alpha, \beta) :=  \mathbb{E}_{(X,Y) \sim  \alpha^{\otimes m } \otimes \beta ^{\otimes m}} [h(X,Y)]
    \label{def:expectationMinibatch}
  \end{equation}
The semi-discrete loss:
\begin{equation}
  U_h(\an,\beta)  := \dbinom{n}{m}^{-1} \sum_{A \in \Pa} \mathbb{E}_{Y \sim \beta ^{\otimes m}} [h(A,Y)]
\label{def:semi_discrete}
\end{equation}
The discrete-discrete loss:
\begin{equation}
U_h(\an,\bn) :=  \dbinom{n}{m}^{-2} \sum_{A \in \Pa} \sum_{B \in \Pb} h(A,B)
\label{def:discrete}
\end{equation}
where $h$ can be the Wasserstein distance $W$, the entropic loss $W_{\varepsilon}$ or the sinkhorn divergence $S_{\varepsilon}$ for a cost $c(\xx,\yy)$.
\end{definition}

Note that $h$ is a U-statistic kernel. Note also that the minibatches elements are drawn without replacement. These quantities represent an average of Wasserstein distance over minibatches of size $m$. Note that samples in $A$ have uniform weights $1/m$ and that the ground cost can be computed between all pair of batches $A$ and $B$. It is easy to see that (\ref{def:discrete}) is an empirical estimator of (\ref{def:expectationMinibatch}). In real world applications, computing the average over all batches is too costly as we have a combinatorial number of batches, that is why we will rely on a subsampled quantity.

\begin{definition}[Minibatch subsampling]\label{def:sub_discrete} Pick an integer $ k > 0$. We define:
  \begin{equation}
    \widetilde{U}_h^k(\an, \bn) :=k^{-1} \sum_{  (A, B)  \in D_k  } h(A, B)
  \end{equation}
where $D_k$ is a set of cardinality $k$ whose elements are drawn at random from the uniform distribution on $ \Gamma:= \mathcal{P}_m( \{X_1, \cdots, X_n  \}) \times \mathcal{P}_m( \{Y_1, \cdots, Y_n \} )  $.
\end{definition}

\begin{comment}
\begin{definition}[Mini-batch Transport]{\blue KF:toujours pas satisfaisant pour moi...}
Let $(A,B)$ drawn without replacement from $ \alpha^{\otimes m } \otimes \beta ^{\otimes m}$, we denote by $\Pi_{ A, B }$ the optimal plan between the random variables. We define the \textit{averaged mini-batch transport matrix}:
\begin{equation}
 \Pi_m(\alpha, \beta) :=  \mathbb{E}_{(A,B) \sim  \alpha^{\otimes m } \otimes \beta ^{\otimes m}}  \Pi_{  A, B } \label{eq:pim}
\end{equation}
When we have discrete probability distributions $\an$ and $\bn$, $A=\{a_1, \dots, a_m\}$ and $B=\{b_1, \dots, b_m\}$ and $ \Pi_{ A, B } \in \mathbb{R}^{n \times n} $. Data's coefficient which are not in $A$ and $B$ are set to 0 and those which are get the optimal transport coefficient.

Following the subsampling idea, we define the subsampled minibatch transportation matrix for $A$ and $B$:
\begin{equation}
 \Pi_k(\an, \bn) :=  k^{-1} \sum_{  (A, B)  \in D_k  }  \Pi_{  A, B }
\end{equation}
\end{definition}
\end{comment}

As the transportation plan might be of interest, let us now review the minibatch definition for the OT plan which can be built for all OT variants which have an OT plan. Formal definitions are provided in appendix.
\begin{definition}[Mini-batch transport plan]%{\rg{RG: nouvelle tentative}}
Consider $\an$ and $\bn$ two discrete probability distributions. For each $A=\{a_1, \dots, a_m\} \in \mathcal{P}_m(\an)$ and $B=\{b_1, \dots, b_m\} \in \mathcal{P}_m(\bn)$ we denote by $\Pi_{ A, B }$ the optimal plan between the random variables, considered as a $n \times n$ matrix where all entries are zero except those indexed in $A \times B$. We define the \textit{averaged mini-batch transport matrix}:
\begin{equation}
 \Pi_m(\an, \bn)  \defas \dbinom{n}{m}^{-2} \sum_{A \in \Pa} \sum_{B \in \Pb} \Pi_{  A, B }. \label{eq:pim}%\defas  \mathbb{E}_{(A,B)}  \Pi_{  A, B }
\end{equation}
%where $(A,B)$ are drawn uniformly on $\mathcal{P}_m(\an) \times \mathcal{P}_m(\bn)$.
Following the subsampling idea, we define the subsampled minibatch transportation matrix for $A$ and $B$:
\begin{equation}
 \Pi_k(\an, \bn) :=  k^{-1} \sum_{  (A, B)  \in D_k  }  \Pi_{  A, B }
\end{equation}
where $D_k$ is drawn as in Definition~\ref{def:sub_discrete}.
%\rg{A-t-on besoin de définir $\Pi_m(\alpha, \beta)$ lorsque $\alpha,\beta$ ne sont pas discretes ?}
\end{definition}

It is well known that the Wasserstein distance suffers from biased gradients \cite{Bellemare_cramerGAN}. We study if $U_h(\an,\bn)$ has a bias \textit{wrt} $U_h(\alpha,\beta)$, and then the bias in $U_h(\an,\bn)$ gradients for first order optimization methods.

\subsection{Illustration on simple examples}
\begin{figure*}[h]
    \centering
    \includegraphics[scale=0.3]{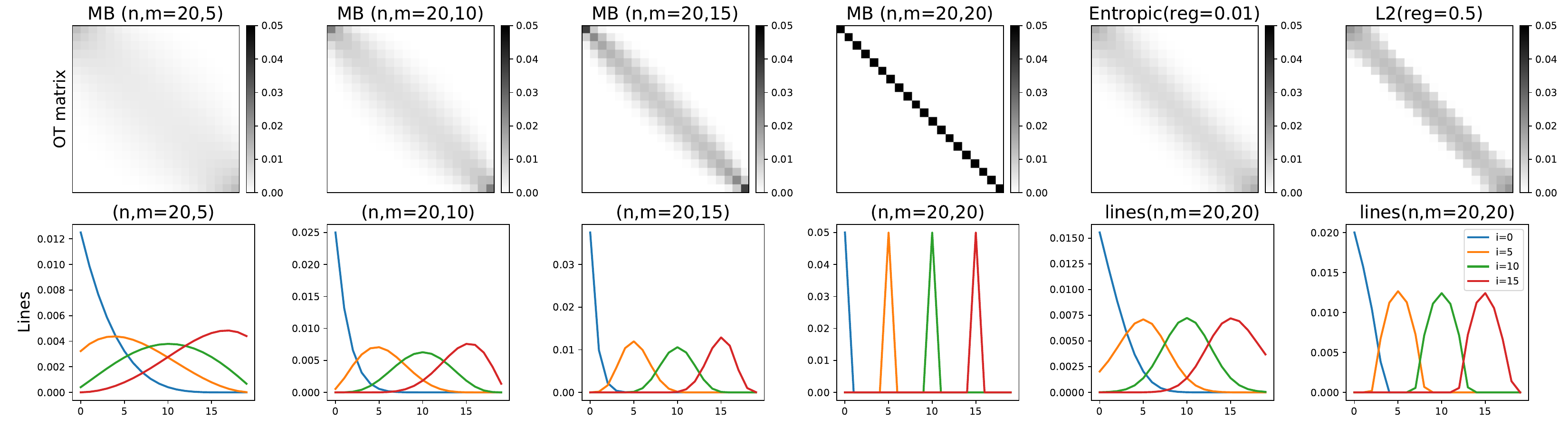}
    \caption{Several OT matrices between distributions with $n=20$ samples in 1D. The first row shows the minibatch OT matrices $\Pi_m$ for different values of $m$, the second row provides the shape of the distributions on the rows of $\Pi_m$. The two last columns correspond to classical entropic and quadratic regularized OT.}
    \label{fig:1D_unif}
\end{figure*}

To illustrate the effect of the minibatch, we compute $\Pi_m$ \eqref{eq:pim} on two simple examples.

\paragraph{Distributions in 1D} The 1D case is an interesting problem because we have access to a closed-form of the optimal transport solution which allows us to calculate the closed-form of a minibatch paradigm. It is the foundation of the sliced Wasserstein distance \cite{Bonnotte2013} which is widely used as an alternative to the Wasserstein distance \cite{liutkus19a,kolouri2016sliced}.

We suppose that we have uniform empirical distributions $\alpha_n$ and $\beta_n$.
We assume (without loss of generality) that the points are ordered in their own distribution. In such a case, we can compute the 1D Wasserstein 1 distance with cost $c(x,y)=|x-y|$ as:
%\begin{equation}
   $ W(\an, \bn) = \frac{1}{n} \sum_{i=1}^n \vert x_i - y_j \vert$ and the OT matrix is simply an identity matrix scaled by $\frac{1}{n}$ (see \cite{COT_Peyre} for more details).
%\end{equation}
 %More details can be found in \cite{COT_Peyre}.
% Here by applying a batch procedure we would have a sum over $m$ terms but the logic remains the same.
 After a short combinatorial calculus (given in appendix A.5), the 1D minibatch transportation matrix coefficient $\pi_{j,k}$ can be computed as $\pi_{j,k}=$:
\begin{align*}
\small{ \frac{1}{m} \dbinom{n}{m}^{-2} \sum_{i=i_{\text{min}}}^{i_{\text{max}}} \dbinom{j-1}{i-1} \dbinom{k-1}{i-1} \dbinom{n-j}{m-i} \dbinom{n-k}{m-i}}
\end{align*}
where $i_{\text{min}} = \text{max}(0, m-n+j, m-n+k)$ and $i_{\text{max}} = \text{min}(j, k)$. $i_{\text{min}}$ and $i_{\text{max}}$ represent the sorting constraints.

%\pa\ref{fig:1D_unif}ragraph{Experiments.}
We show on the first row Figure \ref{fig:1D_unif} the minibatch OT matrices $\Pi_m$ with $n=20$ samples for different value of the minibatch size $m$. We also provide on the second row of the figure a plot of the distributions in several rows of $\Pi_m$. We give the matrices for entropic and quadratic regularized OT for comparison purpose. It is clear from the figure that the OT matrix densifies when $m$ decreases, which has a similar effect as entropic regularization. Note the more localized spread of mass of quadratic regularization that preserve sparsity as discussed in \cite{blondel2018}.
%Let us now consider a toy experiments where we take two uniform probability distributions with 20 samples each. After applying the minibatch strategy for different size of mini-batch, we plot the matrix $\Pi_m$ and a $\Pi_m$'s row to see how the transportation weights are spread over the target samples.
%\paragraph{Interpretation.} On figure \ref{fig:1D_unif}, we see a number \textbf{}of connection between samples growing as the minibatch size decreases. When the number of connection increases, the new built connections have smaller intensities. This phenomenon is similar to regularized OT variants. Indeed, we see that the mass is spread with a L2 regularization or an entropic regularization but differently.
While the entropic regularization spreads the mass in a similar manner for all samples, minibatch OT spreads less the mass on samples at the extremities. Note that the minibatch OT matrices solution is for ordered samples and do not depend on the position of the samples once ordered, as opposed to the regularized OT methods. This will be better illustrated in the next example.

\begin{figure*}[h]
    \centering
    \includegraphics[scale=0.35]{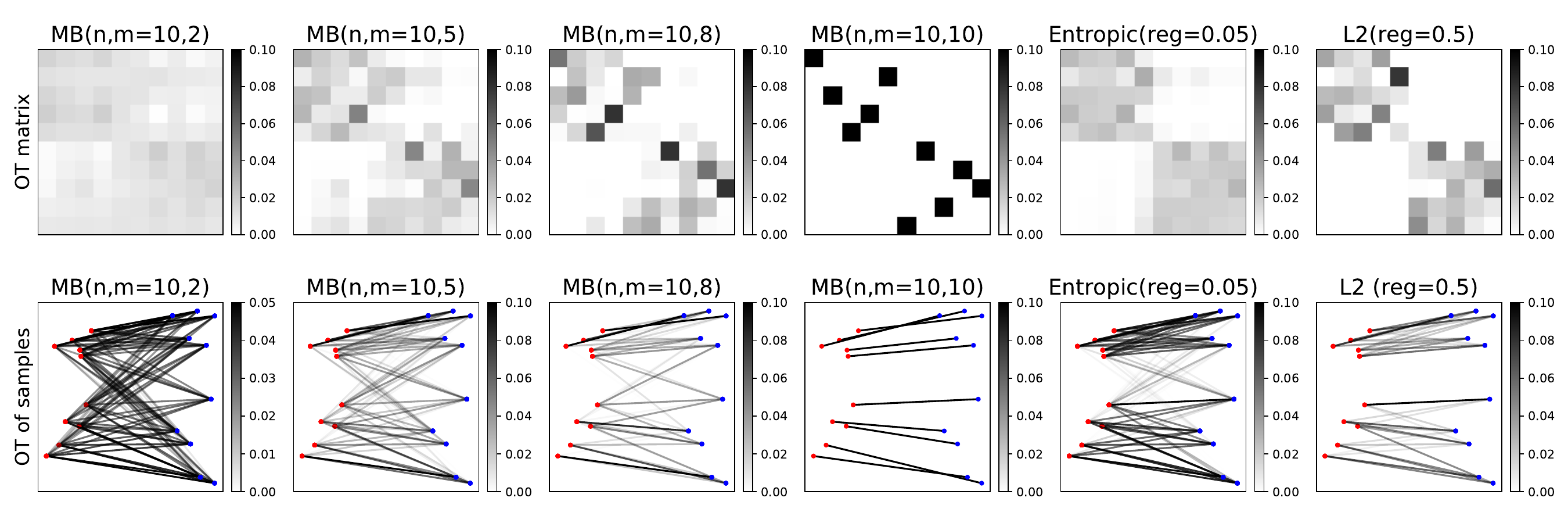}
    \caption{ Several OT matrices between 2D distributions with $n=10$ samples. The first row shows the minibatch OT matrices $\Pi_m$ for different values of $m$, the second row provide the shape of the distributions on the rows of the OT matrices. The second row provide a 2D visualization of where the mass is transported between the 2D positions of the sample.\label{fig:2D_gauss}}

\end{figure*}

\paragraph{Minibatch Wasserstein in 2D} We illustrate the OT matrix between two empirical distributions of 10 samples each in 2D in Figure \ref{fig:2D_gauss}. We use two 2D empirical distributions (point cloud) where the samples have a cluster structure and the samples are sorted \emph{w.r.t.} their cluster.  We can see from the OT matrices in the first row of the figure that the cluster structure is more or less recovered with the regularization effect of the minibatches (and also regularized OT). On the second row one can see the effect of the geometry of the samples on the spread of mass. Similarly to 1D, for Minibatch OT, samples on the border of the simplex cannot spread as much mass as those in the center and have darker rows. This effect is less visible on regularized OT.

%From all those small experiments, we conjecture that a similar phenomenon appears for the transportation matrix in bigger dimensions. Furthermore, we conjecture that the minibatch paradigm acts as a regularization and that for large enough batch size, we can keep a lot of information and get a good computational gain.

\subsection{Basic properties}
We now state some basic properties for minibatch Wasserstein losses. All properties are proved in the appendix. The first property is about the transportation plan $\Pi_m$ between the two initial distributions, defined in \eqref{eq:pim}.
\begin{proposition}
 The transportation plan $\Pi_m(\an, \bn)$ is an admissible transportation plan between the full input distributions $\an, \bn$, and we have : $U_h(\an,\bn) \geq W(\an,\bn)$.
\end{proposition}

The fact that $\Pi_m$ is an admissible transportation plan means that even though it is not optimal, we still do transportation similarly to regularized OT. Note that $\Pi_k$ is not a transportation plan, in general, for a finite $k$ but we study its asymptotic convergence to marginals in the next section. Regarding our empirical estimator, when we have \emph{iid} data, it enjoys the following property:

%It is important to have an admissible transportation plan, because it justifies that we consider all the source distribution is sent to the target distribution. However, note that $\Pi_k$ is not a transportation plan in general. Regarding our empirical estimator, when we have \emph{iid} data, it enjoys the following property:

\begin{proposition}[Unbiased estimator]
$U_h(\an, \bn)$ is an unbiased estimator of $U_h(\alpha, \beta)$ for the continuous setting and of $U_h(\an, \beta)$ for the semi-discrete setting.
\label{prop:bias}
\end{proposition}

As we use minibatch OT for loss function, it is of interest to see if it is still a distance on the distribution space such as the Wasserstein distance or the Sinkhorn divergence.

\begin{proposition}[Positivity and symmetry]
  The minibatch Wasserstein losses are positive and symmetric losses. However, they are not metrics %a metric
  since $U_h(\alpha, \alpha) > 0$.
\end{proposition}

The minibatch Wasserstein losses inherits some properties from the Wasserstein distance %have some distance properties that are inherited from Wasserstein distance,
but the minibatch procedure leads to a strictly positive loss even when starting from unbiased losses such as Sinkhorn divergence or Wasserstein distance.
%This breaks the fundamental \textit{separation axiom}.
%% Remark RG: pas vraiment. L'axiome de séparation me semble plutôt consister à dire que deux points quelconques peuvent être séparés en prenant un voisinage suffisamment petit; ici c'est plutot le contraire : un point n'est même pas son propre voisin.
Remarkably, the Sinkhorn divergence was introduced in the literature to correct the bias from the entropic regularization, and interestingly it was performed in practice on GANs experiments with a minibatch strategy which reintroduced a bias. Whether removing the bias by following the same idea than the Sinkhorn divergence %might lead
leads to a positive loss %and we left the proof as future work.
is an open question left to future work. Furthermore, given the definition of the minibatch losses it is natural to conjecture that they are convex. Informal ingredients towards a proof of this fact are given in the supplementary material.

%, however, we would need to compute the full loss as there are no reason for a resampling loss to remain positive. {\blue KF: Il y avait un bug et c'est $>$ 0 now... Est-ce qu'on mentionne?}.
An important parameter is the value of the minibatch size $m$. We remark that the minibatch procedure allows us to interpolate between OT, when $m=n$ and averaged pairwise distance, when $m=1$. The value of $m$ will also be important for the convergence of our estimator as we will see in the next section.

%\begin{proposition}[Interpolation]
%  The minibatch procedure allows us to interpolate between OT and averaged pairwise distance.
%\end{proposition}
%\rg{Cette proposition m'a l'air un peu vague ou tautologique}

\subsection{Asymptotic convergence}

We are now interested in the asymptotic behavior of our estimator $\widetilde{U}_h^k(\an, \bn)$ and its deviation to $U_h(\alpha,\beta)$. We will give a deviation bound between our subsampled estimator and the expectation (taken on both drawn minibatches and drawn empirical data) of our estimator. This result is given in the continuous setting but a similar result holds for the semi-discrete setting and it follows the same proof. We will give a bound with respect to both $k$ and $n$.

\begin{theorem}[Maximal deviation bound]\label{thm:inc_U_to_mean} Let $\delta \in (0,1) $, $k \geqslant 1$ and $m$ be fixed, and consider two distributions $\alpha,\beta$ with bounded support and an OT loss $h \in \{W, W_\epsilon, S_\epsilon$\}. We have a deviation bound between $\widetilde{U}_h^k(\an, \bn)$ and $U_h(\alpha, \beta)$ depending on the number of empirical data $n$ and the number of batches $k$, with probability at least $1-\delta$ on the draw of $\an,\bn$ and $D_k$ we have:
  \begin{align*}
    \vert \widetilde{U}_h^k(\an, \bn) - U_h(\alpha, \beta) \vert \leq M_h (\sqrt{\frac{ \log(\frac{2}{\delta})}{2\lfloor \frac{n}{m} \rfloor}} + \sqrt{\frac{2\log(\frac{2}{\delta})}{k}  })
    %& \leq M_h (\sqrt{\frac{2 M_{h}^{2}}{\lfloor n/m \rfloor} \ln \frac{2}{\delta}} + \frac{2}{3\lfloor n/m \rfloor} \ln \frac{2}{\delta} + \sqrt{\frac{2 \log(2/\delta)}{k}} )\nonumber
  \end{align*}
%with probability at least 1-$\delta$ and
where
{$M_h$ depends on $h$ and scales at most as $\mathcal{O}(\log(m))$.}
\end{theorem}

This result can be extended with a Bernstein bound (see appendix). The proof is based on two quantities gotten from the triangle inequality. The first quantity is the difference between $U_h(\an, \bn)$ and its expectation $U_h(\alpha, \beta)$. $U_h(\an, \bn)$ is a two-sample U-statistic and we can prove a bound between itself and its expectation in probability \cite{Hoeffding1963}. The second quantity is the difference between $U_h(\an, \bn)$ and the expectation of $\widetilde{U}_h^k(\an, \bn)$. We use the difference between the two quantities to obtain a new random variable quantity. From this new random variable, we use the Hoeffing inequality to obtain a dependence with respect to $k$.

 This deviation bound shows that if we increase the number of data $n$ and batches $k$ while keeping the minibatch size $m$ fixed, we get closer to the expectation. We will investigate the dependence on $k$ and $m$ in different scenarios in the numerical experiments. Remarkably, the bound does not depend on the dimension of $\mathcal{X}$, which is an appealing property when optimizing in high dimension.  %As we loose geometrical information when $m$ decreases, we would like to have a trade-off between the batch size and the fast convergence to $U_h(\alpha, \beta)$.

As discussed before, an interesting output of Minibatch Wasserstein is the minibatch OT matrix $\Pi_m$. Since it is hard to compute in practice, we investigate the error on the marginal constraint of $\Pi_k$.
%As some of our experiments use the transportation matrix, it is of interest to have a similar bound on the distance to the marginales.
In what follows, we denote by $\Pi_{(i)}$ the $i$-th row of matrix $\Pi$ and by %$ \V \V x \V \V_{\infty} = \max_i|x_i| $ the usual infinity norm on vectors. Let us denote by
$ \mathbf{1} \in \R^n $ the vector whose entries are all equal to $1$.
\begin{theorem}[Distance to marginals]\label{thm:dist_marg} Let $ \delta \in (0,1) $, and consider two distributions $\an, \bn$. For all $ k \geqslant 1 $, all $ 1 \leqslant i \leqslant n  $, with probability at least $1-\delta$ on the draw of $\an,\bn$ and $D_k$ we have:
\begin{equation}
\vert  \Pi_k(\an, \bn)_{(i)} \mathbf{1} - \frac{1}{n} \vert \leqslant  \sqrt{\frac{2 \log(2/\delta)}{k}}.
\end{equation}
%with probability at least $1 - \delta$.
\end{theorem}

The proof uses the convergence of $\Pi_k$ to $\Pi_m$ and the fact that %because
$\Pi_m$ is a transportation plan and respects the marginals.

\begin{comment}

\begin{proof}
It is clear to see that the resulting matrix from the minibatch Wasserstein is not the identity. Hence it can not be equal to 0.
\end{proof}
\begin{proof}
When m=n, we get the original OT loss. However when m=1, we have:
\begin{equation}
\mathbb{E}_{(X,Y) \sim  \alpha \otimes \beta } [h(\delta_X,\delta_Y)] = \mathbb{E}_{(X,Y) \sim  \alpha \otimes \beta } [d(X,Y)]
\end{equation}
Where $d(.,.)$ is the associated distance.
\end{proof}
\end{comment}

\subsection{Gradient and optimization}

In this section we review the optimization properties of the minibatch OT losses to ensure the convergence of our loss functions with modern optimization frameworks. We study a standard parametric data fitting problem. Given some discrete samples \(\left(x_{i}\right)_{i=1}^{n} \subset \mathcal{X}\) from some unknown distribution \(\alpha\) , we want to fit a parametric model \(\lambda \mapsto \beta_{\lambda} \in \mathcal{M}(\mathcal{X})\) to \(\alpha\) using the mini-batch Wasserstein distance for a set $\Lambda$ in an Euclidian space.
\begin{equation}
\min _{\lambda \in \Lambda}\quad  U_{h} (\an, \beta_{\lambda})
\end{equation}
%\ \  \text{with  }  U_{W_{\varepsilon}}^{\beta_{\lambda}} =  \frac{1}{\dbinom{n}{m}} \sum_{ \substack{ A \subset S \\ \vert A \vert = m} } \mathbb{E}_{Y_{\lambda} \sim \beta_{\lambda}^{\otimes m}} W_{\varepsilon}(A,Y_{\lambda})
Such problems are written as semi discrete OT problems because one of the distributions is continuous while the other one is discrete. For instance, generative models fall under the scope of such problems \cite{genevay_2018} also known as minimal Wasserstein estimation.
As we have an expectation over one of the distributions, we would like to use a stochastic gradient descent strategy to minimize the problem. By using SGD for their method, \cite{genevay_2018} observed that it worked well in practice and they got meaningful results with minibatches. However it is well known that the empiricial Wasserstein distance is a biased estimator of the Wasserstein distance over the true distributions and leads to biased gradients as discussed in \cite{Bellemare_cramerGAN}, hence SGD might fail. The goal of this section is to prove that unlike the full Wasserstein distance, the minibatch strategy does not suffer from biased gradients. %like the full Wasserstein distance.

\begin{comment}
In our problem, it corresponds to prove that for a subset of the source domain $A$, the estimator is unbiased and that we can exchange expectations and gradients:
\begin{equation}
\nabla_{\lambda} \expect_{Y_{\lambda} \sim \beta_{\lambda}^{\otimes m}}  h(A,Y_{\lambda}) =  \expect_{Y_{\lambda} \sim \beta_{\lambda}^{\otimes m}} \nabla_{\lambda} h(A,Y_{\lambda})
\label{eq:exchange_grad_exp}
\end{equation}
where h is either the entropic loss or the Sinkhorn divergence for two fixed histograms $a,b $.
\end{comment}

As stated in Proposition~\ref{prop:bias}, we enjoy an unbiased estimator. However, the original Wasserstein distance is not differentiable, hence we will, further on, only consider %now consider only
the entropic loss and the Sinkhorn divergence which are differentiable.

\begin{theorem}[Exchange of Gradient and expectation ] Let $\lambda \in V$, where $V$ is a nontrivial open set in $\R^p$. Let $\alpha$ and $\zeta$ be compactly supported distributions. Let $\XX \sim \alpha^{\otimes m}$ and $\ZZ \sim \zeta^{\otimes m}$ be two random variables in $\R^{m \times d}$. Assume $\psi_\lambda: \mathcal{Z} \mapsto \mathcal{Y}$ is differentiable with bounded gradients. Finally, suppose that the ground cost $C$ is $\mathcal{C}^1$. Then we have for the entropic loss and the Sinkhorn divergence:
\begin{align*}
&\nabla_{\lambda} \int_{\mathcal{X}^{\otimes m}}\int_{\mathcal{Z}^{\otimes m}}  h(\XX,\psi_\lambda(\ZZ)) d\alpha^{\otimes m}(\XX) d\zeta^{\otimes m}(\ZZ)\\
&\qquad  =  \int_{\mathcal{X}^{\otimes m}}\int_{\mathcal{Z}^{\otimes m}} \nabla_{\lambda} h(\XX,\psi_\lambda(\ZZ)) d\alpha^{\otimes m}(\XX) d\zeta^{\otimes m}(\ZZ)
\label{app:exchange_grad_exp}
\end{align*}
\end{theorem}

The proof relies on the differentiation lemma. Contrary to the full Wasserstein distance, we proved that the minibatch OT losses do not suffer from biased gradients and this justifies the use of SGD to optimize the problem.

\section{Experiments}
In this section, we illustrate the behavior of minibatch Wasserstein. We use it as a loss function for generative models, use it for gradient flow and color transfer experiments. For our experiments, we relied on the POT package \cite{flamary2017pot} to compute the exact OT solver or the entropic OT loss and the Geomloss package \cite{feydy19a} for the Sinkhorn divergence. The generative model and gradient flow experiments were designed in PyTorch \cite{paszke2017automatic} and all the code is released here \footnote{\url{https://github.com/kilianFatras/minibatch_Wasserstein}}.  %For color transfer experiments, we used the entropic OT solver in \cite{flamary2017pot}.

\subsection{Minibatch Wasserstein generative networks}
We illustrate the use of minibatch Wasserstein loss for generative modeling \cite{goodfellow2014}. The goal is to learn a generative model to generate data close to the target data. We draw 8000 points which follow 8 different gaussian modes (1000 points per mode) in 2D where the modes form a circle. After generating the data, we use a minibatch Wasserstein distance and minibatch Sinkhorn divergence as loss functions with a squared euclidian cost and compared them to WGAN \cite{arjovsky_2017} and its variant with gradient penalty WGAN-GP \cite{Gulrajani2017}. We give implementation details in supplementary.

\begin{figure}[t!]
    \centering
    \includegraphics[width=\columnwidth]{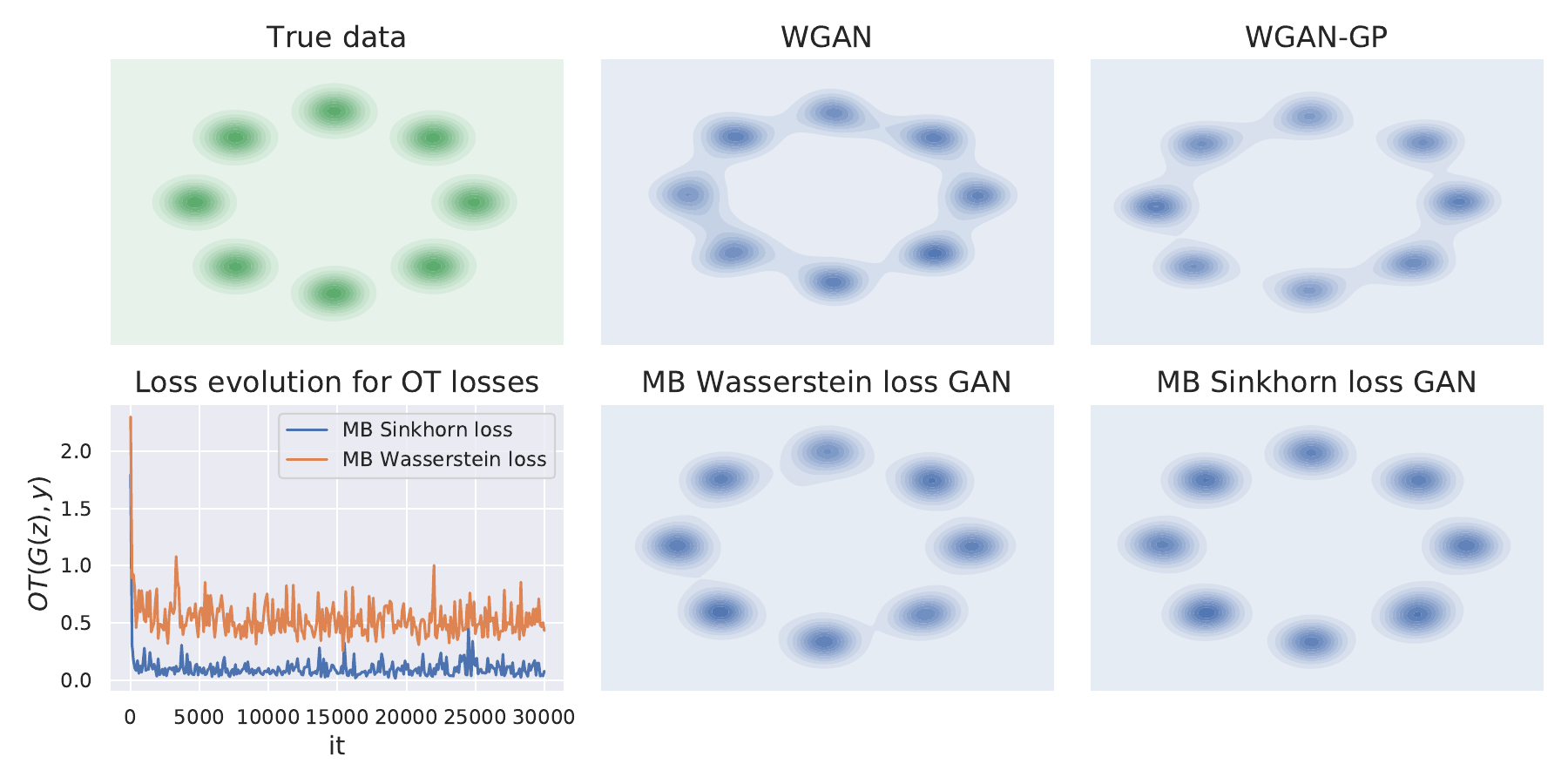}\vspace{-3mm}
    \caption{Generated data in 2D for gaussian modes for several generative models.}
    \label{fig:GAN_Wass}
\end{figure}

We show the estimated 2D distributions in Figure \ref{fig:GAN_Wass}.
For the same architecture it seems that MB Wasserstein trains better generators than WGAN and WGAN-GP. This could come from the fact that MB Wasserstein minimize a complex but well posed objective function (with the squared euclidian cost) while WGAN still need to solve the minmax problem making convergence more difficult especially on this 2D problem.

%\begin{figure*}[h!]
%    \centering
%    \includegraphics[scale=0.27]{sub_part/imgs/GAN/mb_GAN_2D_sinkhorn.png}
%    \caption{Generated data in 2D for gaussian modes with Sinkhorn divergence.}
%    \label{fig:GAN_Sinkhorn}
%\end{figure*}

\begin{figure*}
    \centering
    \includegraphics[scale=0.3]{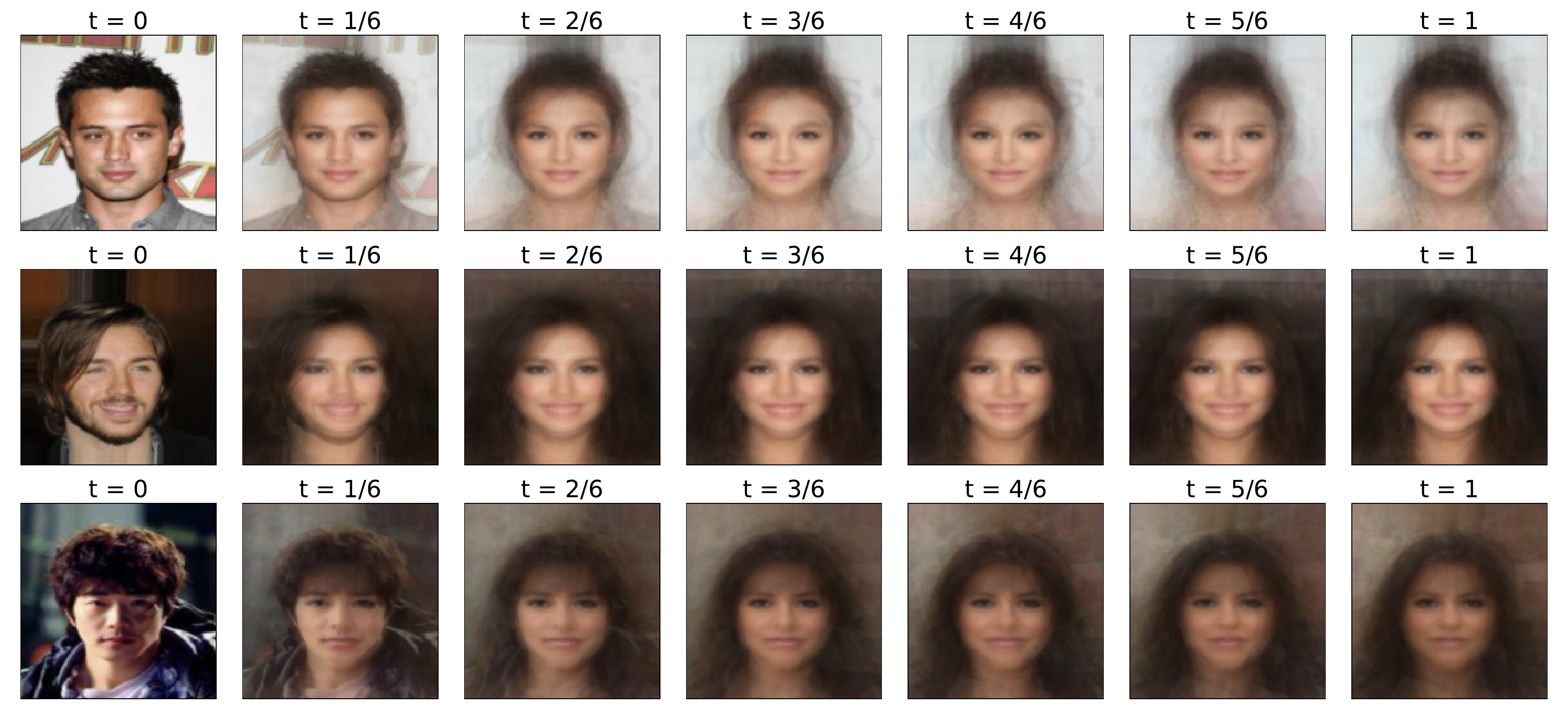}  \vspace{-5mm}
    \caption{Gradient flow on the CelebA dataset. Source data are 5000 male images while target data are 5000 female images. The batch size $m$ is set to 500 and the number of minibatch $k$ is set to 10. The results were computed with the minibatch Wasserstein distance.}
    \label{fig:GF_celeb}
\end{figure*}

\subsection{Minibatch Wasserstein gradient flow}

%\paragraph{Problem setting}
%We now want to study Wasserstein gradient flow experiments \cite{Peyre2015, liutkus19a}.
For a given target distribution $\alpha$, the purpose of gradient flows is to model a distribution $\beta(t)$ which at each iteration follows the gradient direction to minimize the loss $\beta_t \mapsto h(\alpha, \beta_t)$  \cite{Peyre2015, liutkus19a}. %Where $h$ is an OT loss.
The gradient flow simulate the non parametric setting of data fitting problem. In this setting, the modeled distribution $\beta$ is parametrized by a vector $\lambda$ which is the vector position $\xx$ that encodes its support.

We follow the same procedure as in \cite{feydy19a}. The original gradient flow algorithm uses an Euler scheme. Formally, starting from an initial distribution at time $t=0$, it means that at each iteration we integrate the ODE
$$
\dot{\boldsymbol{x}}(t)=- \nabla_{\boldsymbol{x}} F \left(\boldsymbol{x}(t)\right).
$$
In our case, we cannot compute the gradient directly from our minibatch OT losses. As the OT loss inputs are distributions, we have an inherent bias when we calculate the gradient from the weights $\frac{1}{m}$ of samples. To correct this bias, we multiply the gradient by the inverse weight $m$. Finally, for each data $\xx$ we integrate:
\begin{equation}
\dot{\boldsymbol{x}}(t)=-m \nabla_{\boldsymbol{x}}\left[\widetilde{U}_h^k(\an, \bn)\right]\left(\boldsymbol{x}(t)\right)
\end{equation}
We recall that the inherent bias from minibatch makes that the final solution can not be the target distribution.

%\paragraph{Optimization}
%Our loss functions are convex with respect to $\beta$. However, as we do not compute the full loss but a resampled estimate, their is a bias in the gradients. hence to reach convergence, we need to use a SGD paradigm to solve this problem. That is why we will decrease the step size at each iteration.

%\paragraph{Batch size and number of batches}
%We are interested to see the speed convergence et the final distribution's shape with respect to different batch size and different number of averaged batches. We also want to know if the strategy used by \cite{deepjdot, genevay19} was meaningful. Indeed, for their MNIST experiments, they took respectively batch size of 200 and 512 over more than 60000 samples. We want to know the loss of information when we took a such small step size. They also only considered 1 batch, so we would like to know the influence to consider a bigger number of batches.

%\paragraph{Distributions}
\begin{comment}
In 1D, we simply considered two uniform distributions. The $N$ samples were drawn according to those laws.

In 2D, we considered 2 distributions. The modeled distribution had at time $t=0$ the form of a disc, we also added a color palette in order to see if the initial neighbor will still be neighbor at the final time. Regarding the target distribution, we considered a more complex distribution. It has empty slots and a curved tail with a lower density of samples. It will be more challenging to fit the curved tail as the density is smaller.
\end{comment}
The considered data are from the CelebA dataset \cite{liu2015}. We use 5000 male images as source data and 5000 female images as target data. We show the evolution of 3 samples in the source data in Figure \ref{fig:GF_celeb}. We use a squared euclidean cost, a batch size of 500, a learning rate of 0.05 and make 750 iterations. $k$ did not need to be large and was set to 10 in order to stabilize the gradient flow.
%\paragraph{Results}
%While the Sinkhorn divergence over all the distributions is able to perfectly fit the target distribution, this is not the case anymore with minibatches.
We see a natural evolution in the images along the gradient flow similar to results obtained in \cite{liutkus19a}. Interestingly the gradient flow with MB Wasserstein in Figure \ref{fig:GF_celeb} leads to possibly more detailed backgrounds than with MB Sinkhorn (provided in supplementary) probably due to the two layers of regularization in the latter.

\subsection{Large scale barycentric mapping for color transfer}
\begin{figure*}[!h]
    \centering
    \includegraphics[scale=0.35]{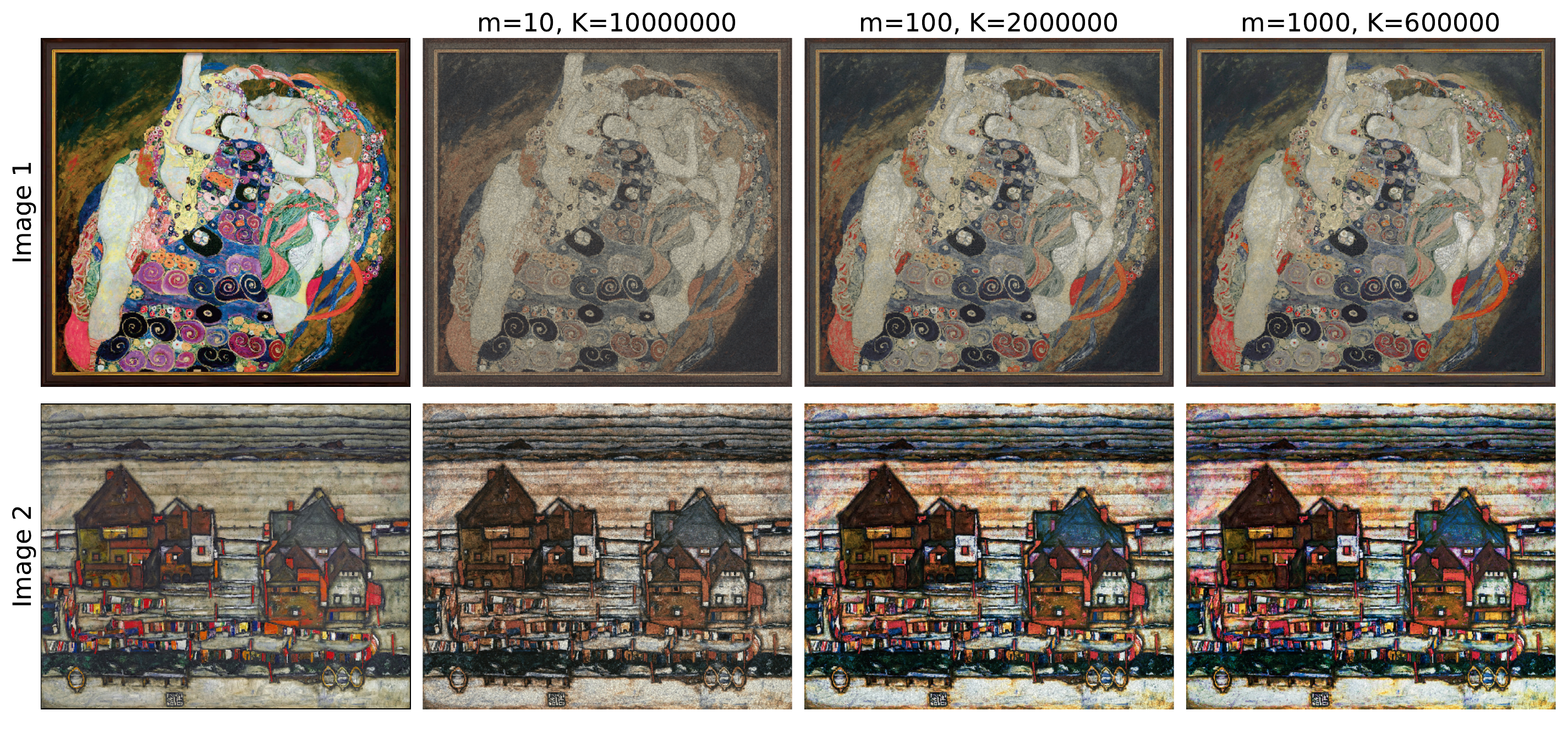}\vspace{-5mm}
    \caption{Color transfert between full images for different batch size and number of batches. (Top) color transfert from image 1 to image 2. (Bottom) color transfert from image 2 to image 1.}
    \label{fig:CT_full_img}
\end{figure*}
% Presentation
%We new present color transfer experiments with the minibatch Wasserstein loss.

The purpose of color transfer is to transform the color of a source image so that it follows the color of a target image. Optimal Transport is a well known method to solve this problem and has been studied before in \cite{Ferradans2013, blondel2018}. Images are represented by point clouds in the RGB color space identified with [0, 1]. Then by calculating the transportation plan between the two point clouds, we get a transfer color mapping by using a barycentric projection. As the number of pixels might be huge, previous work selected a subset of pixels using k-means clusters for each point cloud. This strategy allows to make the problem memory tractable but looses some information. With MB optimal transport, we can compute a barycentric mapping for all pixels in the image by incrementally updating the mapping at each minibtach. %stress out that we do not need to rely on such strategy, we can apply the color transfer on the full images by computing the barycenter mapping incrementally.
When one selects a source batch A and a target batch B, she just needs to update the transformed vector between the considered batches as $Y_s\big\rvert_A = \sum_{B \in \mathcal{P}_m(\beta_n)}  \Pi_{  A, B } X_t\big\rvert_B$. Indeed, to perform the color transfer when we have the full $\Pi_k$ matrix, we compute the matrix product:
\begin{equation}
    Y_s = n_s \Pi_k(\an, \bn) X_t
\end{equation}
that can be computed incrementally by considering restriction to batches (the full algorithm is given in appendix). To the best of our knowledge, it is the first time that a barycentric mapping algorithm has been scaled up to 1M pixel images. About the required memory for experiments, the memory cost to store data is $O(n)$. The minibatch OT calculus requires $O(m^2)$ because we need to store the ground cost and the OT plan. The marginal experiment requires $O(n)$, as we just need to average the marginals of the plan. Finally, the memory cost is $O(n)$ while OT is $O(n^2)$.

% Experimental settings
The source image has (943000, 3) RGB dimension and the target image has RGB dimension (933314, 3). For this experiments, we used the minibatch Wasserstein distance with squared euclidean ground cost for several m and k.%we compare the results between the minibatch framework with the Wasserstein distance for several m and k.
We used batch of size 10, 100 and 1000. We selected $k$ so as to obtain  a good visual quality and observed that a smaller $k$ was needed when using large minibatches. Further experiments which show the dependence on $k$ can be found in appendix. Also note that performing MB optimal transport can be done in parallel and can be greatly speed-up on multi-CPU architectures.
%Unfortunately as $\Pi_m$ has a combinatorial number of terms we can not take the same number of k to reach a good resolution.
%That is why we took different number of batches according to the minibatch size.
%We did not take $k=10^7$ for all the experiments as it would be too long. We performed our experiments on CPU and would like to remind that the computation of our method can be done in parallel.
% Results
 One can see in \ref{fig:CT_full_img} the color transfer (in both directions) provided with our method. We can see that the diversity of colors falls when the batch size is too small as the entropic solver would do for a large regularization parameter. However, even for 1M pixels, a batch size of 1000 is enough to keep a good diversity of colors.

 We also studied empirically the results of theorem \ref{thm:dist_marg}, as shown in Figure \ref{fig:marg_time} we recover the $O(k^{-1/2})$ convergence rate on the marginal with a constant depending on the batch size $m$.  Furthermore, we also empirically studied the computational time and showed that our method is not affected by the number of points with a fixed complexity when an algorithm like Sinkhorn still has a $O(n^2)$ complexity. These experiments show that the minibatch Wasserstein losses are well suited for large scale problems where both memory and computational time are issues.

\begin{figure}[h!]
    \centering
    \includegraphics[scale=0.4]{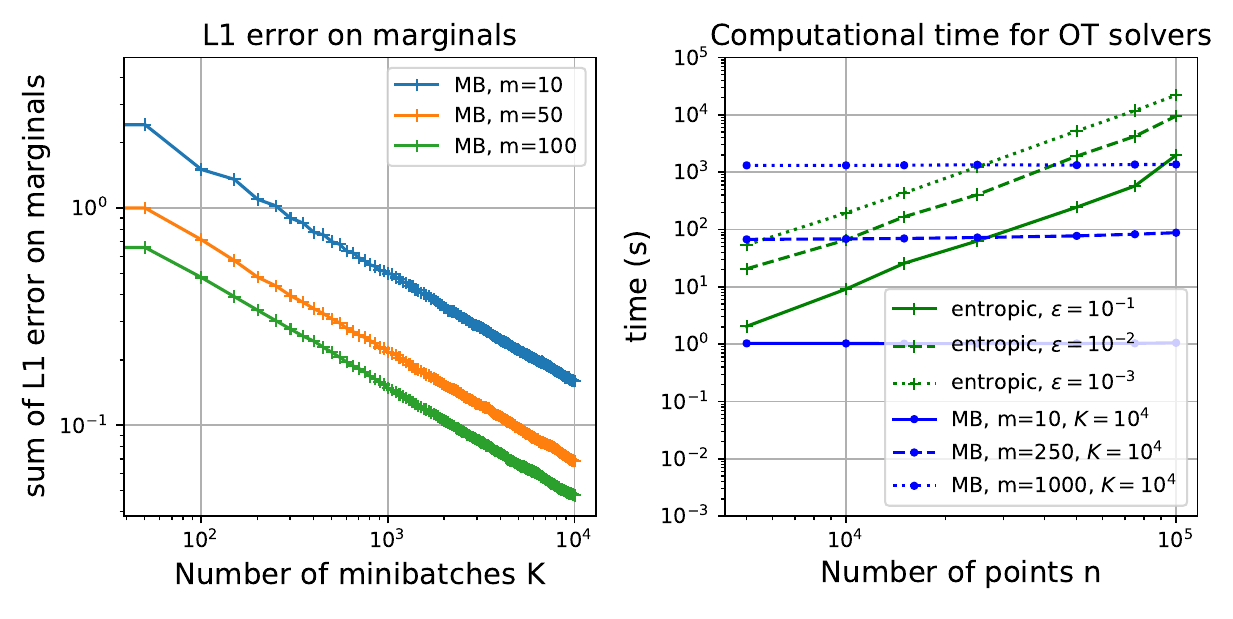}
    \caption{(left) L1 error on both marginals (loglog scale). We selected 1000 points from original images and computed the error on marginals for several m and k (loglog scale). (Right) Computation time for several OT solvers for several number of points in the input distributions, the computation time of the cost matrix is included.}
    \label{fig:marg_time}
\end{figure}

%\begin{figure*}[h!]
%    \centering
%    \includegraphics[scale=0.35]{sub_part/imgs/CT/K_evolution_m_50.png}
%    \caption{Color transfer from MB OT loss with m=50.}
%    \label{fig:CT_m_50}
%\end{figure*}

%\begin{figure}[h!]
    %\centering
    %\includegraphics[scale=0.4]{sub_part/imgs/CT/color_transfer_full_img2.png}
    %\caption{Color transfer on full image from MB Wasserstein distance with sevarel m and K}
    %\label{fig:CT_full_imgs}
%\end{figure}

%\begin{figure*}[h!]
%    \centering
%    \includegraphics[scale=0.5]{sub_part/imgs/CT/transfer_img2_full_m100.png}
%    \caption{Color transfer on full image from MB OT loss with m=100.}
%    \label{fig:CT_m_50}
%\end{figure*}

\section{Conclusion}

In this paper, we studied the impact of using a minibatch strategy in order to reduce the Wasserstein distance complexity. We review the basic properties, and studied the asymptotic behavior of our estimator. We showed a deviation bound between our subsampled estimator and the expectation of our estimator.
%its convergence does not depend on the dimension.
Furthermore, we studied the optimization procedure of our estimator and proved that it enjoys unbiased gradients. Finally, we demonstrated the effect of minibatch strategy with gradient flow experiments, color transfer and GAN experiments. Future works will focus on the geometry of minibatch Wasserstein (for instance on barycenters) and on investigating a debiasing approach similar to the one used for Sinkhorn Divergence.

\subsubsection*{Acknowledgements}

Authors would like to thank Thibault Séjourné and Jean Feydy for fruitful discussions. This work is partially funded through the projects OATMIL ANR-17-CE23-0012 and 3IA Côte d'Azur Investments ANR-19-P3IA-0002 of the French National Research Agency (ANR).

{\small
\nocite{*}
\bibliographystyle{apalike}
\bibliography{egbib}

\begin{thebibliography}{}

\bibitem[Arjovsky et~al., 2017]{arjovsky_2017}
Arjovsky, M., Chintala, S., and Bottou, L. (2017).
\newblock {W}asserstein generative adversarial networks.
\newblock In {\em Proceedings of the 34th International Conference on Machine
  Learning}.

\bibitem[Bassetti et~al., 2006]{Bassetti06}
Bassetti, F., Bodini, A., and Regazzini, E. (2006).
\newblock On minimum kantorovich distance estimators.
\newblock {\em Statistics \& Probability Letters}, 76.

\bibitem[Bellemare et~al., 2017]{Bellemare_cramerGAN}
Bellemare, M.~G., Danihelka, I., Dabney, W., Mohamed, S., Lakshminarayanan, B.,
  Hoyer, S., and Munos, R. (2017).
\newblock The cramer distance as a solution to biased wasserstein gradients.
\newblock {\em CoRR}, abs/1705.10743.

\bibitem[Bernton et~al., 2017]{bernton2017}
Bernton, E., Jacob, P., Gerber, M., and Robert, C. (2017).
\newblock {Inference in generative models using the Wasserstein distance}.
\newblock working paper or preprint.

\bibitem[Blondel et~al., 2018]{blondel2018}
Blondel, M., Seguy, V., and Rolet, A. (2018).
\newblock Smooth and sparse optimal transport.
\newblock In {\em Proceedings of the Twenty-First International Conference on
  Artificial Intelligence and Statistics}.

\bibitem[Bonneel et~al., 2011]{Bonneel2011}
Bonneel, N., van~de Panne, M., Paris, S., and Heidrich, W. (2011).
\newblock Displacement interpolation using lagrangian mass transport.
\newblock In {\em Proceedings of the 2011 SIGGRAPH Asia Conference}, New York,
  NY, USA.

\bibitem[Bonnotte, 2013]{Bonnotte2013}
Bonnotte, N. (2013).
\newblock {\em Unidimensional and Evolution Methods for Optimal
  Transportation}.
\newblock PhD thesis, Université de Paris-Sud.

\bibitem[Bunne et~al., 2019]{bunne19a}
Bunne, C., Alvarez-Melis, D., Krause, A., and Jegelka, S. (2019).
\newblock Learning generative models across incomparable spaces.
\newblock In {\em Proceedings of the 36th International Conference on Machine
  Learning}.

\bibitem[Cl{\'e}men{\c{c}}con, 2011]{clemencccon2011u}
Cl{\'e}men{\c{c}}con, S.~J. (2011).
\newblock On u-processes and clustering performance.
\newblock In {\em Advances in Neural Information Processing Systems}.

\bibitem[Cl{{\'e}}men{\c{c}}on et~al., 2016]{ClemenconJMLR}
Cl{{\'e}}men{\c{c}}on, S., Colin, I., and Bellet, A. (2016).
\newblock Scaling-up empirical risk minimization: Optimization of incomplete
  $u$-statistics.
\newblock {\em Journal of Machine Learning Research}.

\bibitem[Cl{\'e}men{\c{c}}on et~al., 2008]{clemencon2008ranking}
Cl{\'e}men{\c{c}}on, S., Lugosi, G., Vayatis, N., et~al. (2008).
\newblock Ranking and empirical minimization of u-statistics.
\newblock {\em The Annals of Statistics}.

\bibitem[Cl{\'e}men{\c{c}}on et~al., 2013]{clemenccon2013maximal}
Cl{\'e}men{\c{c}}on, S., Robbiano, S., and Tressou, J. (2013).
\newblock Maximal deviations of incomplete u-statistics with applications to
  empirical risk sampling.
\newblock In {\em Proceedings of the 2013 SIAM International Conference on Data
  Mining}.

\bibitem[{Courty} et~al., 2017]{DACourty}
{Courty}, N., {Flamary}, R., {Tuia}, D., and {Rakotomamonjy}, A. (2017).
\newblock Optimal transport for domain adaptation.
\newblock {\em IEEE Transactions on Pattern Analysis and Machine Intelligence}.

\bibitem[Cuturi, 2013]{CuturiSinkhorn}
Cuturi, M. (2013).
\newblock Sinkhorn distances: Lightspeed computation of optimal transport.
\newblock In {\em Advances in Neural Information Processing Systems 26}.

\bibitem[Damodaran et~al., 2018]{deepjdot}
Damodaran, B.~B., Kellenberger, B., Flamary, R., Tuia, D., and Courty, N.
  (2018).
\newblock {DeepJDOT: Deep Joint Distribution Optimal Transport for Unsupervised
  Domain Adaptation}.
\newblock In {\em {ECCV 2018 - 15th European Conference on Computer Vision}}.
  {Springer}.

\bibitem[Ferradans et~al., 2013]{Ferradans2013}
Ferradans, S., Papadakis, N., Rabin, J., Peyr{\'e}, G., and Aujol, J.-F.
  (2013).
\newblock Regularized discrete optimal transport.
\newblock In {\em Scale Space and Variational Methods in Computer Vision}.
  Springer Berlin Heidelberg.

\bibitem[Feydy et~al., 2019]{feydy19a}
Feydy, J., S\'{e}journ\'{e}, T., Vialard, F.-X., Amari, S.-i., Trouve, A., and
  Peyr\'{e}, G. (2019).
\newblock Interpolating between optimal transport and mmd using sinkhorn
  divergences.
\newblock In {\em Proceedings of Machine Learning Research}.

\bibitem[Flamary and Courty, 2017]{flamary2017pot}
Flamary, R. and Courty, N. (2017).
\newblock Pot python optimal transport library.

\bibitem[Frogner et~al., 2015]{frogner_2015}
Frogner, C., Zhang, C., Mobahi, H., Araya, M., and Poggio, T.~A. (2015).
\newblock Learning with a wasserstein loss.
\newblock In {\em Advances in Neural Information Processing Systems 28}.

\bibitem[Genevay et~al., 2019]{genevay19}
Genevay, A., Chizat, L., Bach, F., Cuturi, M., and Peyr\'{e}, G. (2019).
\newblock Sample complexity of sinkhorn divergences.
\newblock In {\em Proceedings of Machine Learning Research}.

\bibitem[Genevay et~al., 2016]{genevay2016stochastic}
Genevay, A., Cuturi, M., Peyr{\'e}, G., and Bach, F. (2016).
\newblock Stochastic optimization for large-scale optimal transport.
\newblock In {\em Advances in neural information processing systems}.

\bibitem[Genevay et~al., 2018]{genevay_2018}
Genevay, A., Peyre, G., and Cuturi, M. (2018).
\newblock Learning generative models with sinkhorn divergences.
\newblock In {\em Proceedings of the Twenty-First International Conference on
  Artificial Intelligence and Statistics}.

\bibitem[Gerber and Maggioni, 2017]{JMLRGerber}
Gerber, S. and Maggioni, M. (2017).
\newblock Multiscale strategies for computing optimal transport.
\newblock {\em Journal of Machine Learning Research}.

\bibitem[Goodfellow et~al., 2014]{goodfellow2014}
Goodfellow, I., Pouget-Abadie, J., Mirza, M., Xu, B., Warde-Farley, D., Ozair,
  S., Courville, A., and Bengio, Y. (2014).
\newblock Generative adversarial nets.
\newblock In {\em Advances in Neural Information Processing Systems 27}.

\bibitem[Gretton et~al., ]{MMD_Gretton}
Gretton, A., Borgwardt, K.~M., Rasch, M.~J., Schlkopf, B., and Smola, A.
\newblock A kernel two-sample test.
\newblock {\em Journal of Machine Learning Research}, 13.

\bibitem[Gulrajani et~al., 2017]{Gulrajani2017}
Gulrajani, I., Ahmed, F., Arjovsky, M., Dumoulin, V., and Courville, A.~C.
  (2017).
\newblock Improved training of wasserstein gans.
\newblock In {\em Advances in Neural Information Processing Systems 30}.

\bibitem[Hoeffding, 1963]{Hoeffding1963}
Hoeffding, W. (1963).
\newblock Probability inequalities for sums of bounded random variables.
\newblock {\em Journal of the American Statistical Association}.

\bibitem[Hull, 1994]{USPS}
Hull, J. (1994).
\newblock Database for handwritten text recognition research.
\newblock {\em Pattern Analysis and Machine Intelligence, IEEE Transactions
  on}, 16.

\bibitem[Hunter, 2007]{Hunter:2007}
Hunter, J.~D. (2007).
\newblock Matplotlib: A 2d graphics environment.
\newblock {\em Computing in Science \& Engineering}, 9(3):90--95.

\bibitem[J~Lee, 2019]{LeeUstats}
J~Lee, A. (2019).
\newblock U-statistics : theory and practice / a. j. lee.
\newblock {\em SERBIULA (sistema Librum 2.0)}.

\bibitem[Kolouri et~al., 2016]{kolouri2016sliced}
Kolouri, S., Zou, Y., and Rohde, G.~K. (2016).
\newblock Sliced wasserstein kernels for probability distributions.
\newblock In {\em Proceedings of the IEEE Conference on Computer Vision and
  Pattern Recognition}.

\bibitem[LeCun and Cortes, 2010]{MNIST}
LeCun, Y. and Cortes, C. (2010).
\newblock {MNIST} handwritten digit database.

\bibitem[Liu et~al., 2015]{liu2015}
Liu, Z., Luo, P., Wang, X., and Tang, X. (2015).
\newblock Deep learning face attributes in the wild.
\newblock In {\em Proceedings of International Conference on Computer Vision
  (ICCV)}.

\bibitem[Liutkus et~al., 2019]{liutkus19a}
Liutkus, A., Simsekli, U., Majewski, S., Durmus, A., and St{\"o}ter, F.-R.
  (2019).
\newblock Sliced-{W}asserstein flows: Nonparametric generative modeling via
  optimal transport and diffusions.
\newblock In {\em Proceedings of the 36th International Conference on Machine
  Learning}.

\bibitem[Mikołaj~Bińkowski, 2018]{Binkowski_DemysMMDGAN}
Mikołaj~Bińkowski, Dougal J.~Sutherland, M. A. A.~G. (2018).
\newblock Demystifying {MMD} {GAN}s.
\newblock {\em International Conference on Learning Representations}.

\bibitem[Papa et~al., 2015]{papa_NIPS2015}
Papa, G., Cl\'{e}men\c{c}on, S., and Bellet, A. (2015).
\newblock Sgd algorithms based on incomplete u-statistics: Large-scale
  minimization of empirical risk.
\newblock In {\em Advances in Neural Information Processing Systems 28}.

\bibitem[Paszke et~al., 2017]{paszke2017automatic}
Paszke, A., Gross, S., Chintala, S., Chanan, G., Yang, E., DeVito, Z., Lin, Z.,
  Desmaison, A., Antiga, L., and Lerer, A. (2017).
\newblock Automatic differentiation in pytorch.

\bibitem[Patrini et~al., 2019]{PatriniBFCBWGN19}
Patrini, G., van~den Berg, R., Forr{\'{e}}, P., Carioni, M., Bhargav, S.,
  Welling, M., Genewein, T., and Nielsen, F. (2019).
\newblock Sinkhorn autoencoders.
\newblock In {\em Proceedings of the Thirty-Fifth Conference on Uncertainty in
  Artificial Intelligence}.

\bibitem[Peyré, 2015]{Peyre2015}
Peyré, G. (2015).
\newblock Entropic approximation of wasserstein gradient flows.
\newblock {\em SIAM Journal on Imaging Sciences}.

\bibitem[Peyré and Cuturi, 2019]{COT_Peyre}
Peyré, G. and Cuturi, M. (2019).
\newblock Computational optimal transport.
\newblock {\em Foundations and Trends® in Machine Learning}.

\bibitem[Seguy et~al., 2018]{seguy2018large}
Seguy, V., Damodaran, B.~B., Flamary, R., Courty, N., Rolet, A., and Blondel,
  M. (2018).
\newblock Large-scale optimal transport and mapping estimation.
\newblock In {\em International Conference on Learning Representations (ICLR)}.

\bibitem[Sommerfeld et~al., 2019]{mbot_Sommerfeld}
Sommerfeld, M., Schrieber, J., Zemel, Y., and Munk, A. (2019).
\newblock Optimal transport: Fast probabilistic approximation with exact
  solvers.
\newblock {\em Journal of Machine Learning Research}.

\bibitem[Weed and Bach, 2019]{weed2019}
Weed, J. and Bach, F. (2019).
\newblock Sharp asymptotic and finite-sample rates of convergence of empirical
  measures in wasserstein distance.
\newblock {\em Bernoulli}.

\bibitem[Wu et~al., 2019]{Wu_2019_CVPR}
Wu, J., Huang, Z., Acharya, D., Li, W., Thoma, J., Paudel, D.~P., and Gool,
  L.~V. (2019).
\newblock Sliced wasserstein generative models.
\newblock In {\em The IEEE Conference on Computer Vision and Pattern
  Recognition (CVPR)}.

\end{thebibliography}
}

\appendix
\onecolumn

{\centering{\LARGE\bfseries Learning with minibatch Wasserstein  : asymptotic and gradient properties}

\vspace{1em}
\centering{{\LARGE\bfseries Supplementary material}}

}

\paragraph{Outline.} The supplementary material of this paper is organized as follows:
\begin{itemize}
    \item In section A, we first review the formalism with definitions, basic property proofs, statistical proofs and optimization proofs. Then we give details about the 1D case.
    \item In section B, we give extra experiments for domain adaptation, minibatch Wasserstein gradient flow in 2D and on the celebA dataset and finally, color transfer.
\end{itemize}

%%%%%%%%%%%%%%%%%%%%%%%%%%%%%%%%%%%%%%%%%%%
%   Rigorous definitions
%%%%%%%%%%%%%%%%%%%%%%%%%%%%%%%%%%%%%%%%%%%

\section{Formalism}
In what follows, without any loss of generality and in order to simplify the notations we will work with the cost matrix $C = C(X,Y) =(\V X_i - Y_i \V )_{1 \leqslant i,j \leqslant n} $.

\subsection{Definitions}
We start giving the formal definitions for the transportation plan $\Pi_m$. We recall that the discrete entropy of a coupling matrix is defined as
$
H(P) = - \sum_{i, j} P_{i, j}\left(\log \left(P_{i, j}\right)-1\right)
$ [chapitre 4, \cite{COT_Peyre}]. The entropic regularization parameter $\varepsilon \in \mathbb{R}_{+}$.
\begin{definition}[Mini-batch Transport] Let $  A \in \mathcal{P}_m(\alpha_n)  $ and $ B \in  \mathcal{P}_m(\beta_n)  $ be two sets. We denote by $ \Pi ^{0} _{ A, B }(\an,\bn) = (\Pi ^{0} _{ A, B }(i,j))_{1\leqslant i,j \leqslant m } \in \R^{m \times m} $ an optimizer of the optimal transport. Formally,
\begin{equation}
\Pi ^{0} _{ A, B }  = \underset{ \Pi \in U( A, B ) }{\operatorname{argmin} } \langle \Pi, C_{\V  A, B}   \rangle - \varepsilon H(\Pi)
\end{equation}
where $ C_{\V  A, B} \in \mathbb{R}^{m \times m} $ is the matrix extracted from $C$ by considering elements of the lines (resp. columns) of $C$ which belong to $ A$ (resp. $ B$) and $H$ the entropy term. $\varepsilon$ is a positive real number that can be equal to 0 to get the original OT problem.
\end{definition}
For two sets $  A \in \mathcal{P}_m(\alpha_n)  $ and $ B \in  \mathcal{P}_m(\beta_n)  $ we denote by $ \Pi_{  A, B }(\an,\bn) \in \mathbb{R}^{n \times n} $ the matrix
\begin{equation}
 \Pi_{  A, B } = ( \Pi ^{0} _{ A, B }(i,j) \mathbf{1}_{A}(i)  \mathbf{1}_{B}(j)   )_{(i,j) \in \alpha_n \times \beta_n }
\end{equation}
\begin{definition}[Averaged mini-batch transport]\label{def:AVG_OT} We define the empirical \textit{averaged mini-batch transport matrix} $  \Pi _m(\an,\bn)$ by the formula
\begin{equation}
 \Pi_m :=  \frac{1}{ \dbinom{n}{m}^2  } \sum_{ \substack{   A \in \mathcal{P}_m(\alpha_n)  } } \sum_{  \substack{  B \in \mathcal{P}_m(\beta_n) } }  \Pi_{  A, B }
\end{equation}
Moreover, we can define the averaged Wasserstein distance over all mini batches as :
  \begin{equation}
    U_W(\an,\bn) = \langle \Pi_m, C \rangle
  \end{equation}
\end{definition}
\begin{remark} Note that this construction is consistent with $U_h(\an,\bn)$.
\end{remark}

%%%%%%%%%%%%%%%%%%%%%%%%%%%%%%%%%%%%%%%%%%%
%		BASIC PROPERTIES
%%%%%%%%%%%%%%%%%%%%%%%%%%%%%%%%%%%%%%%%%%%
\subsection{Basic properties}

\begin{proposition}
 $\Pi_m$ is a transportation plan between the empirical distributions $\an, \bn$.
\end{proposition}
\begin{proof}
We need to verify that the marginals sum to one -e.g. that the sum over any row (resp. column) is equal to $\frac{1}{n}$. Without loss of generality, we will fix a source sample (or row): $i_0$. A simple combinatorial argument gives that $\sum_{  A \in \mathcal{P}_m(\alpha_n)  } \mathbf{1}_{A}(i_0)= \dbinom{n-1}{m-1}$. Now we are ready to sum over the row $i_0$.

  \begin{align}
  \sum_{j=1}^n \Pi_m(i_0,j) %&=   \frac{1}{\dbinom{n}{m}\dbinom{n}{m}} \sum_{j=1}^n \Pi_{A, B}(i_0,j) \\
  &= \frac{1}{\dbinom{n}{m}\dbinom{n}{m}} \sum_{j=1}^n \sum_{ A \in \mathcal{P}_m(\alpha_n) } \sum_{ B \in \mathcal{P}_m(\beta_n)} \Pi_{A, B}(i_0,j) \\
    &= \frac{1}{\dbinom{n}{m}\dbinom{n}{m}} \sum_{ B \in \mathcal{P}_m(\beta_n) } \sum_{j=1}^n \Pi_{A, B}^{0}(i_0,j) \mathbf{1}_{B}(j) \sum_{ A \in \mathcal{P}_m(\alpha_n) } \mathbf{1}_{A}(i_0) \\
    &= \frac{1}{\dbinom{n}{m}\dbinom{n}{m}} \sum_{ B \in \mathcal{P}_m(\beta_n)} \underbrace{\sum_{j=1}^n \Pi_{A, B}^{0}(i_0,j) \mathbf{1}_{B}(j)}_{=1/m} \dbinom{n-1}{m-1} \\
    &= \frac{1}{\dbinom{n}{m}\dbinom{n}{m}} \dbinom{n}{m} \frac{1}{m} \dbinom{n - 1}{m - 1} \\
    %&= \frac{1}{m} \times \frac{\fact{(n - 1)}}{\fact{(m - 1)} \fact{(n - 1 - (m - 1))}} * \frac{\fact{(n - m)} \fact{m}}{\fact{n}}\\
    &= \frac{1}{n}
  \end{align}
The argument is similar for the summation over any column.

\end{proof}

\begin{remark}[Positivity, symmetry and bias] Let $m<n$, the quantity $ U_h $ is positive and symmetric but also stricly positive, i.e $ U_h(\an, \an) >0 $. Indeed,
\begin{align}
U_h(\an, \an) & \defas \frac{1}{ \dbinom{n}{m}^2  } \sum_{ A \in \Pa} \sum_{ A' \in \Pa  } h(A, A') \\
&= \frac{1}{ \dbinom{n}{m}^2  } \sum_{ (A,A') \in \Pa \times \Pb, A \neq A'  } h(A, A') > 0
\end{align}
\end{remark}
\begin{comment}
\begin{proposition}[Convexity]
 The minibatch Wasserstein loss and the minibatch entropic loss are jointly convex with respect to their inputs. The minibatch Sinkhorn divergence is convex with respect to $\alpha$ and with respect to  $\beta$ but not jointly when $\varepsilon > 0$.
\end{proposition}
\begin{proof} Here, for $ k,D \geqslant 1 $ and $A \subset \{ 1, \cdots, m \}$ we denote, for a vector $ x \in (\R^D)^k  $, by $ x_{|A} \in (\R^D)^k $ the vector such that $  (x_{|A})_i = x_i  $ if $i \in \{1, \cdots, m \}$ and $ (x_{|A})_i =0 $ otherwise. Viewing the probability distributions $\an, \bn$ as vectors it is enough to see, since convexity is preserved when
%summing and composing convex functions
summing convex functions and pre-composing them with linear functions
- that for each $A \in \Pa$ and $B \in \Pb$ the restriction map $ (\an, \bn) \mapsto (A, B) $ is linear
and that the map $ (U,V) \mapsto W_{\varepsilon}(U,V) $ is also jointly convex for every $\varepsilon$. The latter is well-known and can be derived immediately thanks to the convexity of the set of transport plans. However $S_{c}^{\varepsilon}$ is not jointly convex for $\varepsilon > 0$, hence the not-jointly convexity for the minibatch Sinkhorn divergence, see section 2.3 \cite{feydy19a}.
\end{proof}
\end{comment}
{\bfseries Convexity}
We introduce a few notations. Let $\mathcal{D}(\R^d) $ be the space defined by
\begin{equation}
 \mathcal{D}(\R^d) := \{  \sum_{i=1}^p \gamma_i \delta_{x_i} : (\gamma_i)_{1 \leq i \leq p}  \in (\R_+)^p, \sum_{i=1}^p \gamma_i = 1;  p \in \mathbb{N}; (x_i)_{1 \leq i \leq p}  \in (\R^d)^p   \}
\end{equation}
It is easy to see that $\mathcal{D}(\R^d) $ is convex. One can actually extend in a natural way the definition of $U_h$ to the set $\mathcal{D}(\R^d) \times \mathcal{D}(\R^d)$. Assuming this can be done, the intuition for convexity is that $U_h$ is an average of convex terms [(section 9.1 and prop 4.6, \cite{COT_Peyre}]. We then claim the convexity of the following maps:
\begin{align*}
    (\alpha_n,\beta_n) \quad & \mapsto U_W(\alpha_n, \beta_n) \\
    \mathcal{D}(\R^d) \times \mathcal{D}(\R^d) & \to \quad \R
\end{align*}
and for $ h = W_{\epsilon} $ or $h = S_{\epsilon}$:
\begin{align*}
    \alpha_n \quad &  \mapsto U_h( \alpha_n,\beta_n) \\
    \mathcal{D}(\R^d) & \to \quad \R \\
    \beta_n & \mapsto U_h(\alpha_n,\beta_n) \\
  \mathcal{D}(\R^d)  & \to \quad \R
\end{align*}
\subsection{Statistical proofs}
Note that because the distributions $ \a$ and $ \beta $ are compactly supported, there exists a constant $M >0$ such that for any $1 \leqslant i,j \leqslant n$, $ \V X_i - Y_j \V \leqslant M $ with $M := \text{diam}(\text{Supp}(\alpha) \cup \text{Supp}(\beta))$.
{ We define the following quantity depending on the OT loss $h$:
  \begin{equation}
    M_h =
    \begin{cases*}
     \text{diam}(\text{Supp}(\alpha) \cup \text{Supp}(\beta))     & if $h = W$ \\
     \frac{3}{2}\left\{ \text{diam}(\text{Supp}(\alpha) \cup \text{Supp}(\beta))  + \varepsilon (2\log_2(m) + 1) \right \}      & if $ h = W_{\varepsilon} $ or $ S_{\varepsilon} $
    \end{cases*}
  \end{equation}

\begin{lemma}[Upper bounds] Let $(A,B) \in \Pa \times \Pb $. We have the following bound for each of the above considered OT losses $h$:
\begin{equation}
  \V   h( A, B ) \V  \leqslant 2 M_h
\end{equation}
\label{app:upper_bound}
\end{lemma}
\begin{proof} We start with the case $ h =W$. Note that with our choice of cost matrix $C=(C_{i,})$ one has $ 0  \leqslant C_{i,j}  \leqslant M_W$. We have for a transport plan $\Pi = (\Pi_{i,j})$ between $A$ and $B$ (with respect to the cost matrix $ C_{|A,B} $)
\[ \V  \langle \Pi,C_{|A,B} \rangle \V \leqslant \sum_{1 \leqslant i,j \leqslant m}  (C_{|A,B})_{ij} \Pi_{i,j} \leqslant M_W  \sum_{1 \leqslant i,j \leqslant m} \Pi_{i,j}  = M_W\]
Hence, $ h(A,B) \leqslant M_W $. \\
If $ h =W_{\varepsilon} $ for an $ \varepsilon>0$. Let us denote by $ E(q) = - \sum_{i=1}^{r} q_i \log(q_i) $ the Shannon entropy of the discrete probability distribution $q =(q_i)_{1 \leqslant i \leqslant r}$. Using the classical fact : $ 0 \leqslant E(q) \leqslant \log_2(r)  $ one estimates for a transport plan $\Pi$:
\[\V \langle \Pi,C_{|A,B} \rangle - \varepsilon H(\Pi) \V \leqslant M_W + \varepsilon (E(\Pi) + 1) \leqslant M_W + \varepsilon (\log_2(m^2) + 1) \leqslant 2 M_h  \]
which gives the intended bound by definition of $ W_{\varepsilon} $.
Lastly, for $h = S_{\varepsilon}$, since it is basically the sum of three terms of the form $ W_{\varepsilon} $ one can conclude.
\end{proof}
}

{\bfseries Proof of Theorem 1} We now give the details of the proof of theorem 1. We start by recalling the definitions of our losses.

\begin{definition}[Minibatch Wasserstein definitions] Given an OT loss $h$ and an integer $m \leq n$, we define the following quantities:

The continuous loss:
  \begin{equation}
    U_h(\alpha, \beta) :=  \mathbb{E}_{(X,Y) \sim  \alpha^{\otimes m } \otimes \beta ^{\otimes m}} [h(X,Y)]
  \end{equation}
  \label{def:expectationMinibatch}
The semi-discrete loss:
\begin{equation}
  U_h(\an,\beta)  := \dbinom{n}{m}^{-1} \sum_{A \in \Pa} \mathbb{E}_{Y \sim \beta ^{\otimes m}} [h(A,Y)]
\label{def:semi_discrete}
\end{equation}
The discrete-discrete loss:
\begin{equation}
U_h(\an,\bn) :=  \dbinom{n}{m}^{-2} \sum_{A \in \Pa} \sum_{B \in \Pb} h(A,B)
\label{def:discrete}
\end{equation}

The subsample discrete-discrete loss. Pick an integer $ k > 0$. We define:
  \begin{equation}
    \widetilde{U}_h^k(\an, \bn) :=k^{-1} \sum_{  (A, B)  \in D_k  } h(A, B)
  \end{equation}
where $D_k$ is a set of cardinality $k$ whose elements are drawn at random from the uniform distribution on $ \Gamma:= \mathcal{P}_m( \{X_1, \cdots, X_n  \}) \times \mathcal{P}_m( \{Y_1, \cdots, Y_n \} )  $. Where $h$ can be the Wasserstein distance $W$, the entropic loss $W_{\varepsilon}$ or the sinkhorn divergence $S_{\varepsilon}$ for a cost $c(\xx,\yy)$.
\end{definition}

\begin{lemma}[U-statistics concentration bound]\label{thm:U_to_mean}
Let $\delta \in (0,1) $ and $m$ be fixed, we have a concentration bound between $U_h(\an, \bn)$ and the expectation over minibatches $U_h(\alpha, \beta)$ depending on the number of empirical data $n$ which follow $\alpha$ and $\beta$.
\begin{equation}
\vert U_h(\an, \bn) - U_h(\alpha , \beta) \vert \leq M_h \sqrt{\frac{ \log(2/\delta)}{2\lfloor n/m \rfloor}}
\end{equation}
with probability at least $1-\delta$. Furthermore, a Bernstein concentration bound is available. Let us denote the variance of the OT loss h over the batches $\sigma_h^2$, \textit{i.e.,} $\sigma_h^2 = Var(h(X_1, \cdots, X_m, Y_1, \cdots, Y_m))$. The variance is bounded by $M_{h}^2$. Then we have with probability at least $\varepsilon$:

%\begin{equation}
%\vert U_h(\an, \bn) - U_h(\alpha , \beta) \vert \leq  \sqrt{\frac{2 \sigma_{h}^{2}}{\lfloor n/m \rfloor} \ln \frac{2}{\delta}} + \frac{2M_h}{3\lfloor n/m \rfloor} \ln \frac{2}{\delta} \leqslant \sqrt{\frac{2 M_{h}^{2}}{\lfloor n/m \rfloor} \ln \frac{2}{\delta}} + \frac{2 M_h}{3\lfloor n/m \rfloor} \ln \frac{2}{\delta}
%\end{equation}
\begin{equation}
    P(\vert U_h(\an, \bn) - U_h(\alpha , \beta) \vert \geq \varepsilon) \leqslant 2 \text{ exp } \left(\frac{-\lfloor n/m \rfloor \varepsilon^2}{2 (\sigma_h^2 + \frac{M_h}{3} \varepsilon)} \right) \leqslant 2 \text{ exp } \left(\frac{-\lfloor n/m \rfloor \varepsilon^2}{2 (M_h^2 + \frac{M_h}{3} \varepsilon)} \right)
\end{equation}
\end{lemma}

\begin{proof}
$U_h(\an, \bn)$ is a two-sample U-statistic and $U_h(\alpha , \beta)$ is its expectation as $\an$ and $\bn$ have \emph{iid} random variables. $U_h(\an, \bn)$ is a sum of dependant variables and Hoeffding found a way to rewrite $U_h(\an, \bn)$ as a sum of independent random variables. As our data are \emph{iid} and our OT loss is bounded, we can apply its third theorem to our U-statistic. The proof can be found in \cite[Section 5]{Hoeffding1963} (the two sample U-statistic case is discussed in 5.b) .
\end{proof}

\begin{lemma}[Deviation bound]\label{app:inc_U_to_U} Let $\an$ and $\bn$ be empirical distributions of respectively $\alpha$ and $\beta$, let $ \delta \in (0,1) $ and $ k \geqslant 1 $. We have a deviation bound between $\widetilde{U}_h^k(\an, \bn)$ and $U_h(\an, \bn )$ depending on the number of batches $k$.

\begin{equation}
\vert  \widetilde{U}_h^k(\an, \bn) - U_h(\an, \bn ) \vert \leqslant M_h \sqrt{\frac{2 \log(2/\delta)}{k}}
\end{equation}
with probability at least $1 - \delta$.
\end{lemma}
\begin{proof}
First note that $\widetilde{U}_h^k(\an, \bn)$ is an incomplete U-statistic of $U_h(\an, \bn )$. Let us consider the sequence of random variables $ ((\mathbf{1}_l(A, B)_{ (A, B) \in \Gamma})_{1 \leqslant l \leqslant k} $ such that $\mathbf{1}_l(A, B)  $ is equal to $1$ if $(A, B) $ has been selected at the $l-$th draw and $0$ otherwise. By construction of $ \widetilde{U}_h^k $, the aforementioned sequence is an i.i.d sequence of random vectors and the $ \mathbf{1}_l(A, B) $ are bernoulli random variables of parameter $ 1/ \vert \Gamma \vert $. We then have
\begin{equation}
 \widetilde{U}_h^k(\an, \bn) - U_h(\an, \bn ) = \frac{1}{k} \sum_{l=1}^k \omega_l
\end{equation}
where $ \omega_l = \sum_{(A, B) \in \Gamma} ( \mathbf{1}_l(A, B) - \frac{1}{ \vert \Gamma \vert }  ) {  h(A,B) } $. Conditioned upon $ X = (X_1, \cdots, X_n)$ and $Y = (Y_1, \cdots, Y_n) $, the variables $ \omega_l $ are independent, centered { and bounded by $2M_h$ thanks to lemma \ref{app:upper_bound}}. Using Hoeffding's inequality yields
\begin{align}
\mathbb{P}( \vert  \widetilde{U}_h^k(\an, \bn) - U_h(\an, \bn ) \vert > \varepsilon  ) & = \mathbb{E} [ \mathbb{P}( \vert  \widetilde{U}_h^k(\an, \bn) - U_h(\an, \bn ) \vert > \varepsilon \vert X,Y )] \\
& = \mathbb{E} [ \mathbb{P}( \vert \frac{1}{k} \sum_{l=1}^k \omega_l ) \vert > \varepsilon \vert X,Y )]\\
& \leqslant  \mathbb{E} [ 2 e^{\frac{-k \varepsilon^2}{2M^2}} ] = 2 e^{\frac{-k \varepsilon^2}{2M^2}}
\end{align}
which concludes the proof.
\end{proof}

\begin{theorem}[Maximal deviation bound]\label{app:inc_U_to_mean} Let $ \delta \in (0,1) $, $k \geqslant 1$ and $m$ be fixed, we have a maximal deviation bound between $\widetilde{U}_h^k(\an, \bn)$ and the expectation over minibatches $U_h(\alpha, \beta)$ depending on the number of empirical data $n$ which follow $\alpha$ and $\beta$ and the number of batches $k$.
  \begin{equation}
    \vert \widetilde{U}_h^k(\an, \bn) - U_h(\alpha, \beta) \vert \leq M_h \sqrt{\frac{ \log(2/\delta)}{2\lfloor n/m \rfloor}} + M_h \sqrt{\frac{2 \log(2/\delta)}{k}}
    %Bernstein : M_h (\kf{\sqrt{\frac{2 M_{h}^{2}}{\lfloor n/m \rfloor} \ln \frac{2}{\delta}} + \frac{2}{3\lfloor n/m \rfloor} \ln \frac{2}{\delta}}
  \end{equation}
with probability at least 1 - $\delta$
\end{theorem}
\begin{proof} Thanks to lemma \ref{app:inc_U_to_U} and \ref{thm:U_to_mean} we get
\begin{align}
\vert \widetilde{U}_h^k(\an, \bn) - U_h(\alpha, \beta) \vert & \leq \vert \widetilde{U}_h^k(\an, \bn) - U_h(\an,\bn) \vert + \vert U_W(\an,\bn)- U_h(\alpha, \beta) \vert \\
%Bernstein : & \leq M_h (\kf{\sqrt{\frac{2 M_{h}^{2}}{\lfloor n/m \rfloor} \ln \frac{2}{\delta}} + \frac{2}{3\lfloor n/m \rfloor} \ln \frac{2}{\delta}}
& \leq M_h \sqrt{\frac{ \log(2/\delta)}{2\lfloor n/m \rfloor}} + M_h \sqrt{\frac{2 \log(2/\delta)}{k}}
\end{align}
with probability at least $ 1 - (\frac{\delta}{2} + \frac{\delta}{2}) = 1 - \delta $. We can get a sharper bound using the Bernstein inequality instead of the Hoeffding inequality as detailed in lemma \label{thm:U_to_mean}.
\end{proof}

{\bfseries Proof of Theorem 2} We now give the details of the proof of theorem 2.
In what follows, we denote by $\Pi_{(i)}$ the $i$-th row of matrix $\Pi$. Let us denote by $ \mathbf{1} \in \R^n $ the vector whose entries are all equal to $1$.
\begin{theorem}[Distance to marginals]\label{thm:dist_to_marg} Let $ \delta \in (0,1) $, we have for all $ k \geqslant 1 $ and all $ 1 \leqslant j \leqslant n  $:
\begin{equation}
\vert  \Pi_k(\an, \bn)_{(i)} \mathbf{1} - \frac{1}{n}  \vert \leqslant  \sqrt{\frac{2 \log(2/\delta)}{k}}
\end{equation}
with probability at least $1 - \delta$.
\end{theorem}
\begin{proof}
We would like to remind that $\Pi_m$ is a transportation plan between the full input distributions $\an$ and $\bn$ and hence, it verifies the marginals, i.e $\Pi_m(\an,\bn)_i \times \mathbf{1} = \frac{1}{n}$. Let us consider the sequence of random variables $ ((\mathbf{1}_p(A, B)_{ (A, B) \in \Gamma})_{1 \leqslant p \leqslant k} $ such that $\mathbf{1}_p(A, B)  $ is equal to $1$ if $(A, B) $ has been selected at the $p-$th draw and $0$ otherwise. By construction of $ \Pi_k(\an,\bn) $, the aforementioned sequence is an i.i.d sequence of random vectors and the $ \mathbf{1}_p(A, B) $ are bernoulli random variables of parameter $ 1/ \vert \Gamma \vert $. We then have
\begin{equation}
\Pi_k(\an,\bn)_{(i)} \mathbf{1}= \frac{1}{k} \sum_{p=1}^k \omega_p
\end{equation}
where $ \omega_p = \sum_{(A, B) \in \Gamma} \sum_{j=1}^n (\Pi_{A,B})_{i,j} \mathbf{1}_p(A, B) $. Conditioned upon $ X = (X_1, \cdots, X_n)$ and $Y = (Y_1, \cdots, Y_n) $, the random vectors $ \omega_p $ are independent, and bounded by $ 1 $. %.\rg{here a look at their variance to use Bernstein ?}.
Moreover, one can observe that $ \mathbb{E}[ \Pi_k(\an,\bn)_i \mathbf{1} ] = \Pi_m(\an,\bn)_i \mathbf{1} $. Using Hoeffding's inequality yields
\begin{align}
\mathbb{P}( \V  \Pi_k(\an,\bn)_i \mathbf{1}  - \Pi_m(\an,\bn)_i \mathbf{1} )  \V > \varepsilon  ) & = \mathbb{E} [ \mathbb{P}(  \V \frac{1}{k} \sum_{p=1}^k \omega_p  - \mathbb{E}[ \frac{1}{k} \sum_{p=1}^k \omega_p]) \V  > \varepsilon  \vert X,Y )] \\
& \leqslant  2 e^{-2k \varepsilon^2}
\end{align}
which concludes the proof.
\end{proof}

%%%%%%%%%%%%%%%%%%%%%%%%%%%%%%%%%%%%%%%%%%%
%		Optimisation
%%%%%%%%%%%%%%%%%%%%%%%%%%%%%%%%%%%%%%%%%%%

\subsection{Optimization}

The main goal of this section is to give a justification of optimization for our minibatch OT losses by giving the \textbf{proof of theorem 3}. More precisely, we show that for the losses $ W_{\varepsilon} $ and $ S_{\varepsilon} $, one can exchange the gradient symbol $ \nabla $ and the expectation $  \mathbb{E} $. It shows for example that a stochastic gradient descent procedure is unbiased and as such legitimate.

\textbf{Main hypothesis.} We define a map $\psi_\lambda: \mathcal{Z} \mapsto \mathcal{Y}$, which is differentiable with bounded gradients. This can model the generator of GANs. We also suppose that the ground cost $C$ is $\mathcal{C}^1$. For a m-tuples of data $\ZZ$, $\psi_{\lambda}(\ZZ)$ denotes the tuple of data $\{\psi_{\lambda}(Z_1), \cdots, \psi_{\lambda}(Z_m)\}$. We introduce the energy map associated to the Wasserstein distance
\[ g: (C,\Pi) \mapsto \langle \Pi, C \rangle - \varepsilon H(\Pi).
\]\\
To prove this theorem, we rely on the "Differentiation Lemma". It allows to exchange expectations and gradients under some hypothesis.

\begin{lemma}[Differentiation lemma]\label{lemma:diff_lebesgue} Let $V$ be a nontrivial open set in $\R^p$ and let $\mathcal{P}$ be a probability distribution on $\R^d$. Define a map $\mathcal{F} : \R^d\times \R^d \times V \rightarrow \R$ with the following properties:
\begin{itemize}
    \item For any $\lambda \in V, \expect_{\mathcal{P}}[\vert \mathcal{F}(X, Y, \lambda)\vert ] < \infty$
    \item For $\mathcal{P}$-almost all $(X, Y) \in \R^d \times \R^d$, the map $V \rightarrow \R$, $\lambda \rightarrow \mathcal{F}(X, Y, \lambda)$ is differentiable.
    \item There exists a $\mathcal{P}$-integrable function $\varphi : \R^d \times \R^d \rightarrow \R$ such that $\vert \partial_{\lambda} \mathcal{F}(X, Y, \lambda) \vert \leq \varphi(X, Y)$ for all $\lambda \in V$.
\end{itemize}

Then, for any $\lambda \in V$ , $E_{\mathcal{P}}[\vert \partial_{\lambda} \mathcal{F}(X, Y, \lambda) \vert ] < \infty$ and the function $\lambda \rightarrow E_{\mathcal{P}}[\mathcal{F}(X, Y, \lambda)]$ is differentiable with differential:
\begin{equation}
  E_{\mathcal{P}} \partial_{\lambda} [\mathcal{F}(X, Y, \lambda)] = \partial_{\lambda} E_{\mathcal{P}} [\mathcal{F}(X, Y, \lambda)]
\end{equation}
\end{lemma}

The differentiation lemma requires the map $\lambda \mapsto h(\XX, \psi_{\lambda}(\ZZ))$ to be differentiable. Thus we need the optimal transport map $h$ to be differentiable \emph{w.r.t} the ground cost $C$. This is achieved by the following lemma:

\begin{lemma}[Danskin, Rockafellar] Let $g: (C,\Pi) \in \R^d \times \R^d \to \R $ be a continuous function. We define $\kappa: C \mapsto \max _{\Pi \in \UU} g(C,\Pi) $ where $\UU \subset \R^d$ is compact. We assume that for each $\Pi \in U$, the function $ g(\cdot, \Pi)  $ is differentiable and that $ \nabla_C g $ depends continuously on $ (C,\Pi)$. If in addition, $g(C,\Pi)$ is convex in $C$, and if $ \overline{C} $ is a point such that $ \operatorname{argmax}_{\Pi \in U}g(\overline{C},\Pi) = \{ \Pi^\star \} $, then $\kappa$ is differentiable at $\overline{C} $ and verifies
\begin{equation}
\nabla \kappa (\overline{C}) = \nabla_C g(\overline{C}, \Pi^\star)
\end{equation}
\label{thm:Danskin1}
\end{lemma}

Hence we have :
\begin{corollary}\label{cor:diff_OT}
For all vectors $\XX, \ZZ \in \R^{dn}$ and integer $\varepsilon>0$, the maps $C \mapsto W_C^\varepsilon$ and $C \mapsto S_C^\varepsilon$ are differentiable \emph{w.r.t} $C$.
\end{corollary}

\begin{proof}
  Let us check that we verify the Danskin hypothesis theorem for entropic regularized Wasserstein distance. First, it is well known that the set of optimal transport plan is a compact set \cite{COT_Peyre}. We recall that the map $g$ is defined as 

  $$g: (C,\Pi) \mapsto \langle \Pi, C \rangle - \varepsilon H(\Pi).$$

  $\langle \Pi, C \rangle $ is bilinear, thus continuous in $(C,\Pi)$, and $H(\Pi)$ is continuous as the product of continuous function. Then $g$ is continuous in $C$ and $\Pi$ as the sum of continuous function. Moreover, the map $g$ is linear in $C$ for all $\Pi$, thus $g$ is $\mathcal{C}^1$ and convex in $C$ for all $\Pi$. Its differential is $\nabla_C g(C,\Pi) = \langle C, \Pi \rangle$ and is linear in both $C$ and $\Pi$. Finally as $\varepsilon > 0$, the map $g$ is strongly convex in $\Pi$ and then the entropic regularized Wasserstein distance has a unique solution $\Pi^\star$ \cite[section 4]{COT_Peyre}. Hence we can apply directly lemma 2 to get that the entropic regularized Wasserstein distance is differentiable for all cost $C$. The Sinkhorn divergence case is direct as it is the sum of three entropic regularized optimal transport terms.
\end{proof}

We are now ready to prove our theorem.

\begin{theorem}[Exchange gradient and expectation] Let $\lambda \in V$, where $V$ is a nontrivial open set in $\R^p$. Let $\alpha$ and $\zeta$ be compactly supported distributions. Let $\XX \sim \alpha^{\otimes m}$ and $\ZZ \sim \zeta^{\otimes m}$ be two random variables in $\R^{m \times d}$. Assume $\psi_\lambda: \mathcal{Z} \mapsto \mathcal{Y}$ is differentiable with bounded gradients. Finally, suppose that the ground cost $C$ is $\mathcal{C}^1$. Then we have for the entropic loss and the Sinkhorn divergence:
%\begin{equation}
%\nabla_{\lambda} \expect_{\YY_{\lambda}}  h(\XX,\YY_{\lambda}) =  \expect_{\YY_{\lambda}} \nabla_{\lambda} h(\XX,\YY_{\lambda})
%\label{app:exchange_grad_exp}
%\end{equation}

\begin{align*}
&\nabla_{\lambda} \int_{\mathcal{X}^{\otimes m}}\int_{\mathcal{Z}^{\otimes m}}  h(\XX,\psi_\lambda(\ZZ)) d\alpha^{\otimes m}(\XX) d\zeta^{\otimes m}(\ZZ)\\
&\qquad  =  \int_{\mathcal{X}^{\otimes m}}\int_{\mathcal{Z}^{\otimes m}} \nabla_{\lambda} h(\XX,\psi_\lambda(\ZZ)) d\alpha^{\otimes m}(\XX) d\zeta^{\otimes m}(\ZZ)
\label{app:exchange_grad_exp}
\end{align*}
\end{theorem}

\begin{proof}
Regarding the Sinkhorn divergence, as it is the sum of three terms of the form $ W_{\varepsilon} $, it suffices to show the theorem for $ h = W_{\varepsilon} $.

The first condition of the Differentiation Lemma is trivial as we have supposed that the random variables $\XX, \ZZ$ have compact supports. Indeed as $\psi_\lambda$ is $\mathcal{C}^1$ and $\zeta$ has compact support, $\mathcal{Y}$ is compact. Hence, the minibatch Wasserstein exists and is bounded on a finite set. We can build a measurable function $\phi$ which takes the biggest cost value $\|C\|_{\infty}$ inside $\mathcal{X}, \mathcal{Y}$ and 0 outside. As $\mathcal{X}, \mathcal{Y}$ are compact, the integral of the function over $\R^d$ is finite.\\

The second hypothesis is also direct by chain rule of differentiable function. Indeed by hypothesis, the map $\lambda \mapsto Y_\lambda$ is differentiable and entropic regularized OT is differentiable thanks to corollary \ref{cor:diff_OT}.

We now check the last hypothesis. Let us write the gradients of $\lambda \mapsto h(\XX, \psi_{\lambda}(\ZZ))$. By chain rule, the gradient of the cell $C_{j,k}$ $\lambda \mapsto C_{j,k}(\XX, \psi_{\lambda}(\ZZ))$ reads 
  \[
  \nabla_{\lambda} C_{j,k}(\XX, \psi_{\lambda}(\ZZ)) =  \nabla_Y C_{j,k}(\XX, \psi_{\lambda}(\ZZ)) \cdot \nabla_{\lambda} \psi_{\lambda}(\ZZ).
  \]
Where $\nabla_Y$ denotes the derivative of the second argument of the cost, \emph{i.e., the gradient of the map $Y \mapsto C(X,Y)$}. Thus, by chain rule of differentiable function, the gradient of $\lambda \mapsto h(\XX, \psi_{\lambda}(\ZZ)))$ reads  
\[
\nabla_{\lambda} h(\XX, \psi_{\lambda}(\ZZ)) = -\text{tr}(\Pi^\star \cdot D^{T})\cdot (\nabla_{\lambda} \psi_{\lambda}(\ZZ)),
\]
where $D_{j,k} = \nabla_{Y} C_{j,k}(\XX, \psi_{\lambda}(\ZZ))$. Hence we now need to dominate $\|\nabla_{\lambda} h(\XX, \psi_{\lambda}(\ZZ))\|.$ As we are in finite dimension, all norms are equivalent and we can consider $\|\cdot\|$ to be a submultiplicative norm. Hence we have :
\begin{align}
\|\nabla_{\lambda} h(\XX, \psi_{\lambda}(\ZZ))\| &\leq \|\Pi^\star \| \|\nabla_{Y} C_{j,k}(\XX, \psi_{\lambda}(\ZZ))\| \|\nabla_{\lambda} \psi_{\lambda}(\ZZ)\|\\
& \leq \frak{m} \|\nabla_{\lambda} \psi_{\lambda}(\ZZ)\|\\
&\text{because we have a $\mathcal{C}^1$ cost with compactly supported $\XX, \psi_{\lambda}(\ZZ)$}\nonumber\\
& \leq \frak{m}^\prime\\
&\text{because we suppose $\lambda \mapsto \psi_{\lambda}$ has bounded gradients}\nonumber
\end{align}
Thus, we can bound the gradient of the map $\lambda \mapsto h(\XX, \psi_{\lambda}(\ZZ))$ by a measurable function which is 0 outside the compacts $\mathcal{X}$ and $\mathcal{Y}$ and some positive constant $\frak{m}^\prime$ inside the compacts. 

Hence, we verify all hypothesis of lemma \ref{lemma:diff_lebesgue} which justifies our claim.
\end{proof}

\subsection{1D case}
We now give the full combinatorial calculus for the 1D case. We start by sorting all the data and give to each of them an index which reprensents their position after the sorting phase. Then we select and sort all the minibatches. $x_j$ can not be at a position superior to its index $j$ inside a batch. For a fixed $x_j$, a simple combinatorial arguments tells you that there are $C_{x_j}^i$ sets where $x_j$ is at the $i$-th position:
\begin{equation}
C_{i, x_j}^{m, n} = \dbinom{j-1}{i-1} \dbinom{n-j}{m-i}
\end{equation}

Suppose that $x_j$ is transported to a $y_k$ points in the target mini batch. Then, they both share the same positions $i$ in their respective minibatch. As there are several $i$ where $x_j$ is transported to $y_k$, we sum over all those possible positions. Hence our current transportation matrix coefficient $\Pi_{j,k}$ can be calculated as :
\begin{equation}
\Pi_{j,k} = \sum_{i=i_{\text{min}}}^{i_{\text{max}}} C_{i, x_j}^{m, n} C_{i, y_k}^{m, n}
\end{equation}

Where  $i_{\text{min}} = \text{max}(0, m-n+j, m-n+k)$ and $i_{\text{max}} = \text{min}(j, k)$. $i_{\text{min}}$ and $i_{\text{max}}$ represent the sorting constraints. Furthermore, as we have uniform weight histograms, we will transport a mass of $\frac{1}{m}$ and averaged it by the total number of transportation. So finally, our transportation matrix coefficient $\Pi_{j,k}$ are:

\begin{equation}
\Pi_{j,k} = \frac{1}{m \dbinom{n}{m}^2} \sum_{i=i_{\text{min}}}^{i_{\text{max}}} C_{i, x_j}^{m, n} C_{i, y_k}^{m, n}
\end{equation}

\newpage
\section{Extra experiments}
In this section, we present extra experiments on the utility of using minibatch Wasserstein loss for domain adaptation, gradient flow and color transfer. We also give the algorithm which computes the barycentric mapping incrementally.

\subsection{Generative models}

We give implementation details of our batch Wasserstein generative models. We use a normal Gaussian noise in a latent space of dimension 10 and the generator is designed as a simple multilayer perceptron with 2 hidden layers of respectively 128 and 32 units with ReLu activation functions, and one final layer with 2 output neurons. For the different OT losses, the generator is trained with the same learning rate equal to 0.05. The optimizer is the Adam optimizer with $\beta_1=0$ and $\beta_2=0.9$. For the Sinkhorn divergence we set $\varepsilon$ to 0.01. For WGAN and WGAN-GP we train a discriminator with the same hidden layers than the generator. We update the discriminator 5 times before one update of the generator. WGAN is trained with RMSprop optimizer and WGAN-GP with Adam optimizer ($\beta_1=0$, $\beta_2=0.9$) as done in their original papers. The learning rate is set to $10^{-4}$ for both. WGAN-GP has a gradient penalty parameter set to 10. All models are trained for 30000 iterations with a batch size of 100. Our minibatch OT losses use $k=1$, which means that we compute the stochastic gradient on only one minibatch, and larger $k$ was not needed to get meaningful results.

\subsection{Domain adaptation}
% Problem description
Domain adaptation problems consist to transfer knowledge from a source domain to a target domain. The goal is to use the labeled data in the source domain in order to classify the unlabeled data in the target domain. \cite{DACourty} used optimal transport to transport the source data to the target data by computing an OT map. Then they used a barycentric mapping to transport the source data to the target domain with their label. Optimal transport has been successful on this problem and we now want to study the impact of the minibatch OT losses and different OT variants.

%Experimental setting
We consider two common datasets for domain adaptation problems : MNIST \cite{MNIST} and USPS \cite{USPS}. The datasets are composed of hand written digits betwenn 0 and 9. MNIST have 60000 training samples and USPS have 7291 training samples. We select 7000 samples from each dataset. The used cost for those experiments is a normalized squared euclidean cost. We want to study the number of samples which are transported on same labeled data from the source dataset to the target dataset. That is why we will study the proportion of mass between same labeled data in the transportation matrix.

The experiments use minibatch Wasserstein loss. We will use several k and m values, while for the entropic OT loss we will consider values of epsilon between $10^{-3}$ and $1$. For each $m$ and $k$, we conducted the experiments 10 times and we plot the mean and standard deviation for each $m$ and $k$.

\begin{figure}[ht!]
    \centering
    \includegraphics[scale=0.5]{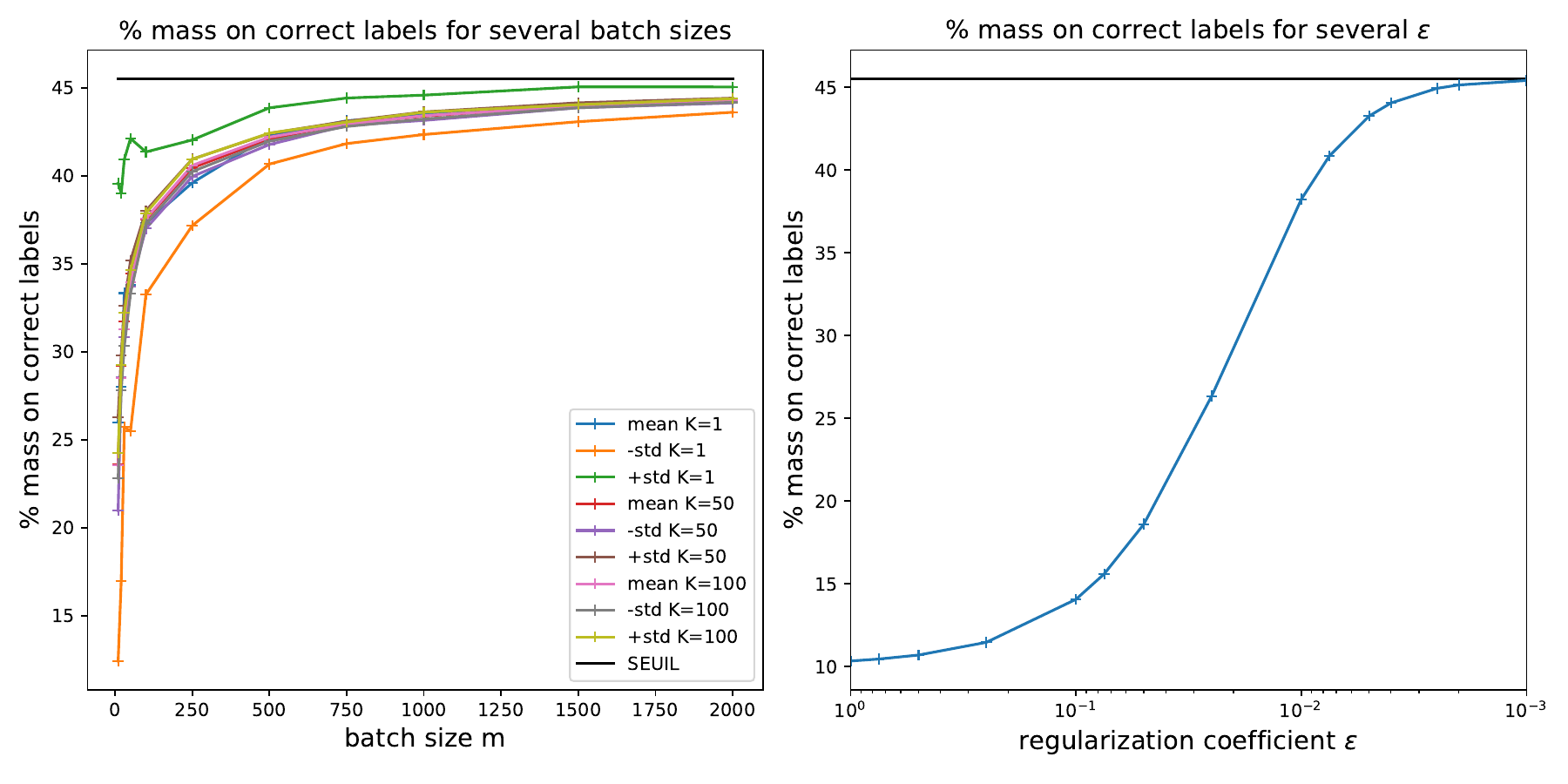}
    \caption{Proportion of correct transfered data between S/T domains for OT MB.}
    \label{fig:proportion_MB}
\end{figure}

This experiment shows that considering a very small batch size hurts the number of images transported on correct labels and taking a large number of batches does not correct the performance. We also see that the number of batches $k$ reduces the variance and should decrease when the batch size increases. Furthermore, we see that when $m$ decreases, we have a similar performance than for the entropic OT loss with a large regularization parameter $\varepsilon$. We conjecture, that doing the minibatch entropic loss with a large $\varepsilon$ parameter can lead to over regularization and can hurt the performance.

\subsection{Minibatch Wasserstein gradient flow}
We experimented the minibatch OT gradient flow to distributions in 2D. The purpose is to see the relevance of minibatch Wasserstein gradient flow for shape matching applications. We used the same experiments as in \cite{feydy19a} and relied on the geomloss package. In 2D we selected 500 data points following the image's pixel distribution. The experiments were conducted with the minibatch Wasserstein loss. We observe that we are not able to recover the target distribution, it is expected as our loss is strictly positive. However, for large enough batch size, the final distribution fits almost perfectly the target distribution and our loss leads to a good approximation.

Nevertheless we can see that taking a batch size too small results in a loss of information and drives the data toward the high density area as pointed in the 2D experiments. Regarding the number of minibatches $k$, it does not influence the shape of the final distribution.

\begin{comment}
\begin{figure}[ht!]
\centering
\begin{minipage}{.4\textwidth}
  \centering
  \includegraphics[scale=0.2]{sub_part/imgs/GF/gradient_flow_1D.png}
  \label{fig:test1}
\end{minipage}%
\begin{minipage}{.6\textwidth}
  \centering
  \includegraphics[scale=0.225]{sub_part/imgs/GF/gradient_flow_2D.png}
  \label{fig:test2}
\end{minipage}
\caption{(a) Gradient flow between 1D distributions for several batch sizes and several k. The source and the target distributions have 150 samples each and we ploted the KDE of each distribution. (b) Gradient flow between 2D distributions for several batch sizes and several k. The source and the target distributions have 500 samples each.}
\end{figure}
\end{comment}

\begin{figure}[!h]
    \centering
  \includegraphics[scale=0.27]{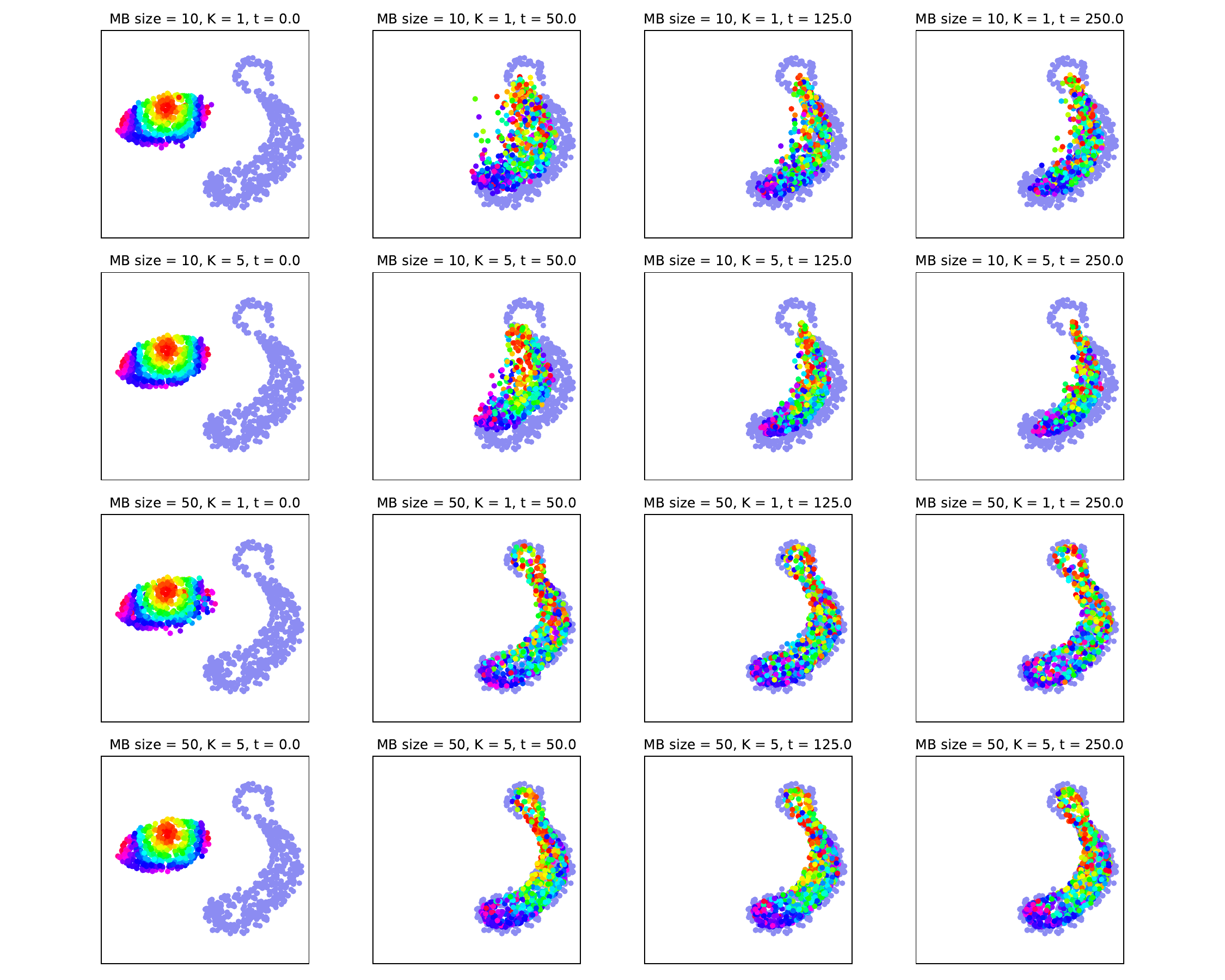}
    \caption{Gradient flow between 2D distributions for several batch sizes $m$ and several number of batches $k$. The source and the target distributions have 500 samples each.}
    \label{fig:GF_2D}
\end{figure}

Regarding the gradient flow on the celebA dataset, we now show the results when we use the minibatch Sinkhorn divergence instead of the minibatch Wasserstein distance. The minibatch Sinkhorn divergence is slower in practice than the minibatch Wasserstein distance and the samples converge toward different pictures. However, we can still see a natural evolution in the images along the gradient flow.
\begin{figure}[!h]
    \centering
    \includegraphics[scale=0.27]{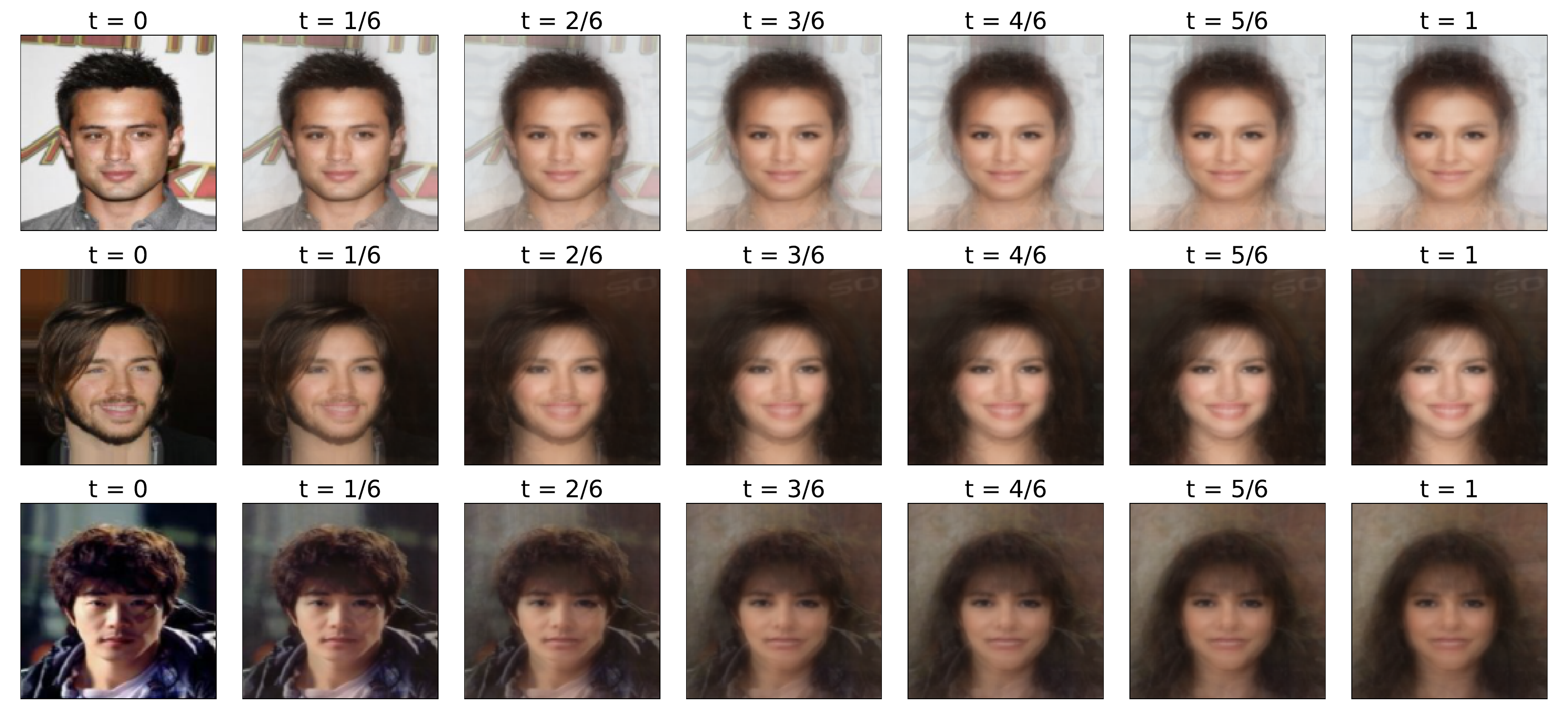}
    \caption{Gradient flow on the CelebA dataset. Source data are 5000 male images while target data are 5000 female images. The batch size $m$ is set to 500 and the number of minibatch $k$ is set to 10. The results were computed with the minibatch Sinkhorn divergence.}
    \label{fig:GF_celeb_sinkhorn}
\end{figure}

\subsection{Color transfer between subset of images}

In order to present the influence of $k$ for barycentric mapping, we present extra experiments for color transfer. We compute a k-means clustering with $l$ clusters for each point cloud. For each image, we computed 1000 k-means clusters of the point clouds and applied the optimal transport algorithms between those subsets. We consider batch size of 10, 50 and 100. We show the color transfer for each image for $k= 5000$ and $k=20000$ batches.

In what follows, we present the algorithm which computes the color transfer vectors incrementally without requiring the storage of the full cost matrix neither the full transportation matrix $\Pi_k$.

\begin{algorithm}[h]\label{alg:CT_incremental}
 \caption{Computation of incremental color transfer}
\SetAlgoLined
%\KwResult{Update of the full $/pi$ matrix}
{\bfseries Inputs:} $m$, $k$, source domain $\boldsymbol{X_s} \in \mathbb{R}^{n\times d}$, target domain $\boldsymbol{X_t} \in \mathbb{R}^{n\times d}$ \;
%{\bfseries Preprocessing} : ground cost C \;
{\bfseries Results} : $\boldsymbol{Y_s}$, $\boldsymbol{Y_t}$ \;
{\bfseries Initialisation} : $\boldsymbol{Y_s} \in \mathbb{R}^{n\times d}$, $\boldsymbol{Y_t} \in \mathbb{R}^{n\times d}$\;
\For{t=1, $\cdots$, k}{
Select a set $A$ of $m$ samples in $\boldsymbol{X_s}$\;
Select a set $B$ of $m$ samples in $\boldsymbol{X_t}$\;
Compute the restricted cost $C_{A, B}$\;
$G \leftarrow \underset{\Pi \in U(A, B)}{\text{argmin}} \langle C_{A, B}, \Pi \rangle $\;
$\boldsymbol{Y_s}\big\rvert_{A} \leftarrow  \boldsymbol{Y_s}\big\rvert_{A} + G . \boldsymbol{X_t}\big\rvert_{B}$\;
$\boldsymbol{Y_t}\big\rvert_{B} \leftarrow \boldsymbol{Y_t}\big\rvert_{B} + G^T . \boldsymbol{X_s}\big\rvert_{A}$\;
 %$\Pi_{t+1}\big\rvert_{A, B} = \underset{\Pi \in U(A, B)}{argmin} \langle C_{A, B}, \Pi \rangle\ + \Pi_t \big\rvert_{A, B}$\;
 }
return $\frac{n}{k} \boldsymbol{Y_s}$, $\frac{n}{k} \boldsymbol{Y_t}$
\end{algorithm}

\begin{figure}[!h]
    \centering
    \includegraphics[scale=0.38]{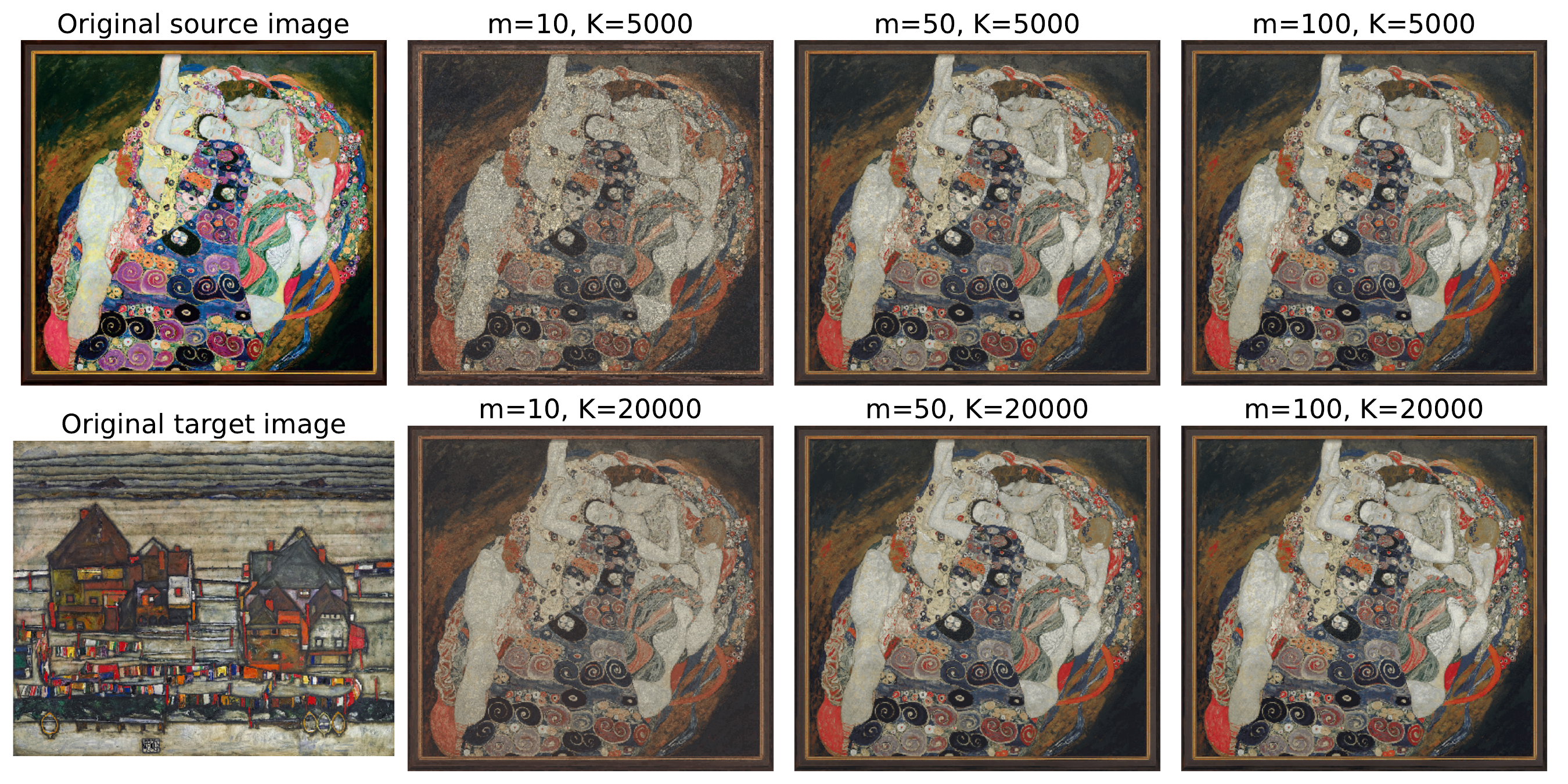}
    \caption{Color transfer from MB Wasserstein loss for several m and K. The minibatch Wasserstein distance is computed between subsets of original images.}
    \label{fig:CT_m_100}
\end{figure}

We see that for each batch size $m$, when the number of batches $k$ increases, we get better resolution for our images. It is expected as our matrix $\Pi_k$ gets closer to $\Pi_m$. However, when $m$ is small, we will need to have a large $k$ to get good resolutions for images. We can see this phenomenon for $m=10$, where $k=5000$ was not enough to have a good resolution. However, $k=5000$ was enough to get good resolutions for $m=1000$.

\end{document}